\documentclass[pdflatex,sn-apa]{sn-jnl}

\usepackage{graphicx}
\usepackage{multirow}
\usepackage{amsmath,amssymb,amsfonts}
\usepackage{amsthm}
\usepackage{mathrsfs}
\usepackage[title]{appendix}
\usepackage{xcolor}
\usepackage{textcomp}
\usepackage{manyfoot}
\usepackage{booktabs}
\usepackage{algorithm}
\usepackage[noend]{algpseudocode}
\algnewcommand{\algorithmicbreak}{\textbf{break}}
\algnewcommand{\algorithmiccontinue}{\textbf{continue}}
\algnewcommand{\algorithmicforeach}{\textbf{for each}}
\algnewcommand\Break{\algorithmicbreak}
\algnewcommand\Continue{\algorithmiccontinue}
\algnewcommand{\LeftComment}[1]{\Statex \(\triangleright\) #1}
\algdef{SE}[FOR]{ForEach}{EndForEach}[1]
	{\algorithmicforeach\ #1\ \algorithmicdo}
	{\algorithmicend\ \algorithmicforeach}
\algtext*{EndForEach} 
\algdef{SE}[DOWHILE]{DoWhile}{EndDoWhile}{\algorithmicdo}[1]{\algorithmicwhile\ #1}
\usepackage[capitalise, nameinlink]{cleveref}
\makeatletter
	\newcounter{HALG@line}
	\renewcommand{\theHALG@line}{\thealgorithm.\arabic{ALG@line}}
\makeatother
\usepackage{acronym}
\usepackage{csquotes}
\usepackage{paralist}
\usepackage{tikz}
\usetikzlibrary{angles, arrows, arrows.meta, automata, backgrounds, calc, patterns, positioning, quotes, shapes, shadows}
\usepackage{subcaption}

\theoremstyle{thmstyleone}
\newtheorem{theorem}{Theorem}
\newtheorem{corollary}{Corollary}

\newtheorem{proposition}[theorem]{Proposition}
\newtheorem{example}{Example}

\theoremstyle{thmstylethree}
\newtheorem{definition}{Definition}

\crefname{algocf}{Alg.}{Algs.}
\Crefname{algocf}{Algorithm}{Algorithms}
\crefname{algorithm}{Alg.}{Algs.}
\Crefname{algorithm}{Algorithm}{Algorithms}
\crefname{corollary}{Cor.}{Cors.}
\Crefname{corollary}{Corollary}{Corollaries}
\crefname{definition}{Def.}{Defs.}
\Crefname{definition}{Definition}{Definitions}
\crefname{equation}{Eq.}{Eqs.}
\Crefname{equation}{Equation}{Equations}
\crefname{example}{Ex.}{Exs.}
\Crefname{example}{Example}{Examples}
\crefname{line}{Line}{Lines}
\Crefname{line}{Line}{Lines}
\crefname{proposition}{Prop.}{Props.}
\Crefname{proposition}{Proposition}{Propositions}
\crefname{section}{Sec.}{Secs.}
\Crefname{section}{Section}{Sections}
\crefname{theorem}{Thm.}{Thms.}
\Crefname{theorem}{Theorem}{Theorems}

\acrodef{aacp}[$\alpha$-ACP]{$\alpha$-Advanced Colour Passing}
\acrodef{acp}[ACP]{Advanced Colour Passing}
\acrodef{bn}[BN]{Bayesian network}
\acrodef{cbn}[CBN]{causal Bayesian network}
\acrodef{cfg}[CFG]{causal factor graph}
\acrodef{cp}[CP]{Colour Passing}
\acrodef{crv}[CRV]{counting randvar}
\acrodef{decor}[DECOR]{Detection of Commutative Factors}
\acrodef{deft}[DEFT]{Detection of Exchangeable Factors}
\acrodef{eacp}[$\varepsilon$-ACP]{$\varepsilon$-Advanced Colour Passing}
\acrodef{elci}[ELCI]{Extended Lifted Causal Inference}
\acrodef{fg}[FG]{factor graph}
\acrodef{kld}[KLD]{Kullback-Leibler divergence}
\acrodef{lci}[LCI]{Lifted Causal Inference}
\acrodef{lifagu}[LIFAGU]{Lifting Factor Graphs with Some Unknown Factors}
\acrodef{ljt}[LJT]{Lifted Junction Tree}
\acrodef{lv}[logvar]{logical variable}
\acrodef{lve}[LVE]{Lifted Variable Elimination}
\acrodef{mln}[MLN]{Markov logic network}
\acrodef{mn}[MN]{Markov network}
\acrodef{mrf}[MRF]{Markov random field}
\acrodef{pcfg}[PCFG]{parametric causal factor graph}
\acrodef{pdpcfg}[PD-PCFG]{partially directed parametric causal factor graph}
\acrodef{pf}[parfactor]{parametric factor}
\acrodef{pfg}[PFG]{parametric factor graph}
\acrodef{prv}[PRV]{parameterised randvar}
\acrodef{rv}[randvar]{random variable}
\acrodef{ve}[VE]{variable elimination}
\acrodef{wl}[WL]{Weisfeiler-Leman}

\newcommand{\alginput}[1]{\hspace*{\algorithmicindent} \textbf{Input:} #1}
\newcommand{\algoutput}[1]{\hspace*{\algorithmicindent} \textbf{Output:} #1}

\newcommand{\abs}[1]{\ensuremath{\lvert #1 \rvert}}

\newcommand{\domain}[1]{\ensuremath{\mathrm{dom}(#1)}}
\newcommand{\gr}{\ensuremath{\mathrm{gr}}}

\newcommand{\lv}{\ensuremath{\mathrm{lv}}}

\newcommand{\nupmodels}{\mathrel{\not\!\perp\!\!\!\perp}}
\newcommand{\range}[1]{\ensuremath{\mathrm{range}(#1)}}

\newcommand{\upmodels}{\mathrel{\perp\!\!\!\perp}}

\newcommand{\low}{\ensuremath{\mathrm{low}}}

\newcommand{\high}{\ensuremath{\mathrm{high}}}

\newcommand{\Ch}{\ensuremath{\mathrm{Ch}}}
\newcommand{\De}{\ensuremath{\mathrm{De}}}
\newcommand{\Ne}{\ensuremath{\mathrm{Ne}}}
\newcommand{\Pa}{\ensuremath{\mathrm{Pa}}}
\newcommand{\pa}{\ensuremath{\mathrm{pa}}}

\definecolor{myyellow}{RGB}{247,192,26}
\definecolor{myblue}{RGB}{37,122,164}
\definecolor{mygreen}{RGB}{78,155,133}
\definecolor{mypurple}{RGB}{86,51,94}

\definecolor{newblue}{RGB}{50,113,173}
\definecolor{newred}{RGB}{222,32,36}
\definecolor{newgreen}{RGB}{70,165,69}
\definecolor{newpurple}{RGB}{140,69,152}

\definecolor{cborange}{RGB}{230,159,0}
\definecolor{cbblue}{RGB}{30,136,229}
\definecolor{cbbluedark}{RGB}{46,37,133}
\definecolor{cbpurple}{RGB}{170,68,153}
\definecolor{cbgreen}{RGB}{0,77,64}
\definecolor{cbgreenlight}{RGB}{93,168,153}
\definecolor{cbbrown}{RGB}{126,41,84}

\pgfdeclarelayer{bg}
\pgfsetlayers{bg,main}

\tikzset{
	rv/.style={draw, ellipse},
	pf/.style={draw, rectangle, fill = gray!30},
	arc/.style = {->, >={[round,sep]Stealth}},
}

\newcommand\factorat[4]{
	\node[pf, label={#2:{#3}}](#4) at (#1) {};
}

\newcommand\factor[6]{
	\node[pf, #1=#3 of #2, label={#4:{#5}}](#6) {};
}

\newcommand\pfs[8]{
	\node[pf, #1=#3 of #2, xshift=-1mm, yshift=1mm](#6) {};
	\node[pf, #1=#3 of #2, label={[label distance=1mm]#4:{#5}}](#7) {};
	\node[pf, #1=#3 of #2, xshift=1mm, yshift=-1mm](#8) {};
}

\newcommand\pfsat[6]{
	\node[pf, xshift=-1mm, yshift=1mm](#4) at (#1) {};
	\node[pf, label={[label distance=1mm]#2:{#3}}](#5) at (#1) {};
	\node[pf, xshift=1mm, yshift=-1mm](#6) at (#1) {};
}

\begin{document}

\title[Lifted Causal Inference]{Lifted Causal Inference}

\author*[1,2]{\fnm{Malte} \sur{Luttermann}}\email{malte.luttermann@dfki.de}
\author[3]{\fnm{Tanya} \sur{Braun}}\email{tanya.braun@uni-muenster.de}
\author[1]{\fnm{Ralf} \sur{Möller}}\email{ralf.moeller@uni-hamburg.de}
\author[1]{\fnm{Marcel} \sur{Gehrke}}\email{marcel.gehrke@uni-hamburg.de}

\affil[1]{\orgdiv{Institute for Humanities-Centered Artificial Intelligence}, \orgname{University of Hamburg}, \orgaddress{\city{Hamburg}, \country{Germany}}}
\affil*[2]{\orgname{German Research Center for Artificial Intelligence (DFKI)}, \orgaddress{\city{Lübeck}, \country{Germany}}}
\affil[3]{\orgdiv{Data Science Group}, \orgname{University of Münster}, \orgaddress{\city{Münster}, \country{Germany}}}

\abstract{%
	Lifted inference exploits indistinguishabilities in probabilistic graphical models by using a representative for indistinguishable objects, thereby speeding up query answering while maintaining exact answers.
	In this article, we show how lifting can be applied to efficiently compute causal effects in relational domains.
	More specifically, we introduce \acp{pcfg} to incorporate causal knowledge in lifted models and give a formal semantics of interventions therein.
	We further present the \ac{lci} algorithm to compute causal effects on a lifted level, thereby drastically speeding up causal inference compared to propositional inference, e.g., in \aclp{cbn}.
	In addition, we present \acp{pdpcfg} as a generalisation of \acp{pcfg} to handle partial causal knowledge and extend \ac{lci} to perform lifted causal inference in a \ac{pdpcfg}, thereby extending the applicability of lifted causal inference to a broader range of models requiring less prior knowledge about causal relationships.
}

\keywords{causal inference, lifting, probabilistic relational models}

\maketitle

\acresetall

\section{Introduction}
A fundamental problem in the research field of artificial intelligence for an intelligent agent is to plan and act rationally in a relational domain.
To compute the best possible action in a perceived state, the agent considers the available actions and chooses the one with the maximum expected utility.
When computing the expected utility of an action performed on a specific variable, it is crucial to deploy the semantics of an intervention instead of a typical conditioning on that variable~\citep[Chapter~4]{Pearl2009a}.
When calculating the effect of an intervention, a specific variable is set to a fixed value and all incoming probabilistic causal influences of this variable must be ignored for the specific query.
It is fundamental to deploy the semantics of an intervention instead of the typical conditioning to correctly determine the effect of an action.
Otherwise, when treating actions as evidence (by applying a classical conditioning), conclusions might become misleading.
For example, assume a scenario in which the severity of fires influences the number of firefighters trying to extinguish the fire, that is, the more severe a fire is, the more firefighters are on duty.
Classical conditioning then suggests to reduce the number of firefighters to reduce the severity of fires (because the probability for a severe fire is lower when observing a low number of firefighters on duty).
In this article, we apply lifting to efficiently compute causal effects (and hence, the correct effect of actions) in relational domains, where efficient inference refers to inference running in polynomial time with respect to domain sizes.

Over the last years, causal models have become a widely used formalism to answer questions concerning the causal effect of an intervention on a \ac{rv} on another \ac{rv}.
A causal model consists of
\begin{inparaitem}
	\item[(i)] a causal graph representing the causal relationships between the involved \acp{rv}, and
	\item[(ii)] a probability distribution over the \acp{rv}.
\end{inparaitem}
There has been a considerable amount of work to perform causal effect estimation in causal models, and most of this work focuses on propositional models~\citep{Spirtes2000a,Pearl2009a,Pearl2016a,Peters2017a}.
Some works extend propositional (undirected) \acp{fg} by adding edge directions to enable the computation of the effect of interventions~\citep{Frey2003a,Winn2012a}.
\citet{Maier2013a} introduce so-called relational causal models to express causal dependencies within relational domains.
Their work focuses on causal discovery, that is, on learning relational causal models from observed data~\citep{Maier2010a}.
Further developments on relational causal models also focus on causal discovery and on reasoning about conditional independence (e.g., \citealp{Lee2015a,Lee2016a,Lee2019a}).
Relational causal models provide a lifted representation (that is, a representation that abstracts over individual objects and hence over all instantiations of a relational model) to reason about conditional independence, however, relational causal models do not support lifted causal inference.
More recently, relational causal models have also been extended to cover cyclic dependency structures~\citep{Ahsan2022a,Ahsan2023a}.
Prior work dealing with the estimation of causal effects in relational domains still applies propositional probabilistic inference~\citep{Arbour2016a,Salimi2020a}.
Consequently, there is a lack of efficient algorithms to compute causal effects on a lifted level.
In probabilistic inference, lifting exploits indistinguishabilities in a relational model, allowing to carry out query answering more efficiently while maintaining exact answers~\citep{Niepert2014a}.
First introduced by \citet{Poole2003a}, \acp{pfg} and \ac{lve} allow to perform lifted probabilistic inference, resulting in significant speed-ups for probabilistic query answering in relational domains.
Over time, \ac{lve} has been refined by many researchers to reach its current form~\citep{DeSalvoBraz2005a,DeSalvoBraz2006a,Milch2008a,Kisynski2009a,Taghipour2013a,Braun2018a}.
To perform efficient inference in a \ac{pfg} not only for single queries but also for sets of queries, \citet{Braun2016a} introduce the \ac{ljt} algorithm.
\Acp{pfg} have been well-studied for many years and have been developed further to incorporate probabilistic inference over time~\citep{Gehrke2018a,Gehrke2020a}, and, among other extensions, to allow for decision making by following the maximum expected utility principle~\citep{Gehrke2018b,Gehrke2019c,Braun2022a}.
Markov logic networks are another lifted representation and have been extended to incorporate maximum expected utility as well~\citep{Apsel2012a}.
In this article, we extend \acp{pfg} to enable lifted causal inference to correctly determine the effect of actions on a lifted level.

This article is based on and extends the works \citep{Luttermann2024b} and \citep{Luttermann2024g}.
Specifically, we present the introduced models and algorithms for lifted causal inference under a unified view, thereby making the following contributions:
First, we give a formal definition of \acp{cfg} as an extension of \acp{fg} to incorporate causal knowledge on a propositional level.
We then provide a unified view on fully directed lifted causal models introduced by \citet{Luttermann2024b} and partially directed lifted causal models introduced by \citet{Luttermann2024g}.
In particular, we expose the connection between these models and their corresponding algorithms to perform lifted causal inference therein.
We especially highlight the differences in the assumptions made in the two models and exhibit how these assumptions affect their corresponding inference algorithms.
Furthermore, we align the model definitions and algorithm descriptions for consistency of terminology and improved clarity.
We also extend the theoretical results for fully directed and partially directed lifted causal models and showcase all presented concepts on a full running example.

The remaining part of this article is structured as follows.
In \cref{sec:pcfg_cfg}, we introduce \acp{cfg} and define the notion of an intervention in a \ac{cfg} to allow for the computation of causal effects therein (on a propositional level).
Thereafter, in \cref{sec:pcfg_pcfg}, we present \acp{pcfg} as an extension of \acp{pfg} and provide a formal semantics of interventions in \acp{pcfg}.
By incorporating causal knowledge on a lifted level, a \ac{pcfg} allows to perform lifted causal inference, thereby enabling efficient decision making in relational domains using the notion of an intervention.
Then, in \cref{sec:pcfg_lci}, we elucidate the \ac{lci} algorithm, which operates on a \ac{pcfg}, and show how \ac{lci} computes causal effects on a lifted level to avoid grounding the \ac{pcfg} as much as possible.
We then portray \acp{pdpcfg} as a generalisation of \acp{pcfg} in \cref{sec:pdpcfg_pdpcfg}.
Afterwards, we investigate how the effect of interventions can be computed in a \ac{pdpcfg} in the presence of unknown causal relationships.
In \cref{sec:pdpcfg_lci}, we present the \ac{elci} algorithm as a generalisation of \ac{lci} to efficiently compute causal effects in a \ac{pdpcfg} before we conclude this article in \cref{sec:pdpcfg_conclusion}.

\section{Causal Factor Graphs} \label{sec:pcfg_cfg}
Similar to a \ac{cbn}~\citep{Pearl1988a,Pearl2009a}, a \ac{cfg} is a probabilistic graphical model that simultaneously encodes a probability distribution over a set of \acp{rv} $\boldsymbol R$ and causal relationships between the \acp{rv} in $\boldsymbol R$.
As in non-causal \acp{fg}~\citep{Frey1997a,Kschischang2001a}, the full joint probability distribution is encoded as a product of factors, where each factor is a function of a subset of the \acp{rv}.
The difference between an \ac{fg} and a \ac{cfg} is that a \ac{cfg} contains directed edges instead of undirected edges to represent the causal relationships between the \acp{rv}.
More specifically, a directed edge from a \ac{rv} $R_i$ to another \ac{rv} $R_j$ in a \ac{cfg} indicates that $R_i$ is a direct cause of $R_j$ and thus, the value of $R_i$ influences the value of $R_j$~\citep{Pearl2009a}.
Therefore, in any causal graph, it holds that the value of a \ac{rv} depends on the values of its parents.
We next provide a formal definition of a \ac{cfg} based on the definition of directed \acp{fg} given by \citet{Frey2003a}.
In the following, we denote by $\range{R_i}$ the range of a \ac{rv} $R_i$, that is, the set of possible values that $R_i$ can take.

\begin{definition}[Causal Factor Graph] \label{def:pcfg_cfg}
	We define a \emph{\ac{cfg}} as a tuple $M = (\boldsymbol V, \boldsymbol E, \boldsymbol \Phi)$ where $(\boldsymbol V, \boldsymbol E)$ is a directed bipartite graph with node set $\boldsymbol V = \boldsymbol R \cup \boldsymbol F$ and edge set $\boldsymbol E \subseteq \boldsymbol R \times \boldsymbol F$ and $\boldsymbol \Phi$ is a set of function definitions.
	The set of nodes $\boldsymbol V$ is divided into a set of \acp{rv} $\boldsymbol R = \{ R_1, \ldots, R_n \}$ (variable nodes) and a set of function names (factor nodes) $\boldsymbol F = \{ f_1, \ldots, f_m \}$.
	Every function name $f_j \in \boldsymbol F$ has a function definition (factor, for short) $\phi_j(\mathcal R_j) \in \boldsymbol \Phi$, where $\phi_j \colon \times_{R \in \mathcal R_j} \range{R} \mapsto \mathbb{R}_{\geq 0}$ maps range values of a sequence $\mathcal R_j$ of \acp{rv} from $\boldsymbol R$ to a non-negative real number (potential).
	For each function definition, there must be at least one sequence of range values that is mapped to a potential which is non-zero.
	The set of edges $\boldsymbol E$ contains two types of edges.
	For every factor node $f_j \in \boldsymbol F$ with corresponding function definition $\phi_j(\mathcal R_j)$, there is either an undirected edge $\{ R_i, f_j \} \in \boldsymbol E$ or a directed edge $(f_j, R_i) \in \boldsymbol E$ for every \ac{rv} $R_i \in \mathcal R_j$.
	We stipulate that for every factor node $f_j \in \boldsymbol F$, there exists exactly one outgoing directed edge $(f_j, R_i) \in \boldsymbol E$ among the edges incident to $f_j$.
	Each directed edge $\{ R_i, f_j \}, (f_j \to R_k)$ from a \ac{rv} $R_i \in \boldsymbol R$ to a \ac{rv} $R_k \in \boldsymbol R$ via a factor node $f_j \in \boldsymbol F$ corresponds to a direct causal relationship between $R_i$ and $R_k$.
	Furthermore, $M$ has to be acyclic, that is, $M$ is required to not contain any directed cycles.
	The joint potential for an assignment $\boldsymbol R = \boldsymbol r$ is defined as the product over all factors in the \ac{cfg} $M$:
	\begin{align}
		\psi_M(\boldsymbol R = \boldsymbol r) = \prod_{j=1}^{m} \phi_j(\mathcal R_j = \boldsymbol r_j),
	\end{align}
	where $\boldsymbol r_j$ is a projection of $\boldsymbol r$ to the argument list of $\phi_j$.
	The full joint probability distribution $P_M(\boldsymbol R)$ over $\boldsymbol R$ encoded by $M$ is then given by the normalised joint potential:
	\begin{align} \label{eq:pcfg_cfg_joint_distribution}
		P_M(\boldsymbol R = \boldsymbol r)
		&= \frac{1}{Z} \psi_M(\boldsymbol R = \boldsymbol r),
	\end{align}
	where the normalisation constant $Z$ is defined as the sum of all joint potentials:
	\begin{align}
		Z
		&= \sum_{\boldsymbol r \in \range{R_1} \times \ldots \times \range{R_n}} \psi_M(\boldsymbol R = \boldsymbol r).
	\end{align}
\end{definition}

\begin{example}[Causal Factor Graph] \label{ex:pcfg_cfg}
	Consider the \ac{cfg} $M = (\boldsymbol V, \boldsymbol E, \boldsymbol \Phi)$ depicted in \cref{fig:pcfg_example_cfg_graph}.
	$M$ represents the causal relationships between the competences and salaries of three employees $Alice$, $Bob$, and $Charlie$ and the revenue of the company they work for.
	The underlying causal graph is illustrated in \cref{fig:pcfg_example_cbn_graph}.
	In set notation, the graph structure of $M$ is given by $(\boldsymbol V = \boldsymbol R \cup \boldsymbol F, \boldsymbol E)$, where
	\begin{align*}
		\boldsymbol R = &\{ ComA, ComB, ComC, Rev, SalA, SalB, SalC \}, \\
		\boldsymbol F = &\{ f_1, f_2, f_3, f_4, f_5, f_6, f_7 \},~\text{and} \\
		\boldsymbol E = &\{ (f_1, ComA), \{ComA, f_4\}, \{ComA, f_5\}, (f_2, ComB), \{ComB, f_4\}, \{ComB, f_6\}, \\
		&\phantom{\{} (f_3, ComC), \{ComC, f_4\}, \{ComC, f_7\}, (f_4, Rev), \{Rev, f_5\}, \{Rev, f_6\}, \{Rev, f_7\}, \\
		&\phantom{\{} (f_5, SalA), (f_6, SalB), (f_7, SalC) \}.
	\end{align*}
	Moreover, the set of function definitions is
	\begin{align*}
		\boldsymbol \Phi = &\{ \phi_1(ComA), \phi_2(ComB), \phi_3(ComC), \phi_4(ComA, ComB, ComC, Rev), \\
		&\phantom{\{} \phi_5(ComA, Rev, SalA), \phi_6(ComB, Rev, SalB), \phi_7(ComC, Rev, SalC) \},
	\end{align*}
	where we omit the exact specification of the potential tables for brevity.
\end{example}

\begin{figure}[t]
	\centering
	\begin{subfigure}[t]{0.49\linewidth}
		\centering
		\begin{tikzpicture}[rv/.append style={minimum height=2.2em, minimum width=5.2em}]
			\node[rv, draw] (ca) {$ComA$};
			\node[rv, draw, right = 0.5em of ca] (cb) {$ComB$};
			\node[rv, draw, right = 0.5em of cb] (cc) {$ComC$};
			\node[rv, draw, below = 7.5em of ca] (sa) {$SalA$};
			\node[rv, draw, below = 7.5em of cb] (sb) {$SalB$};
			\node[rv, draw, below = 7.5em of cc] (sc) {$SalC$};
			\node[rv, draw] (rev) at ($(cb)!0.5!(sb)$) {$Rev$};
		
			\factor{above}{ca}{0.75em}{180}{$f_1$}{f1}
			\factor{above}{cb}{0.75em}{180}{$f_2$}{f2}
			\factor{above}{cc}{0.75em}{180}{$f_3$}{f3}
			\factorat{$(cb)!0.5!(rev)$}{180}{$f_4$}{f4}
			\factor{above}{sa}{1.0em}{180}{$f_5$}{f5}
			\factor{above}{sb}{1.0em}{180}{$f_6$}{f6}
			\factor{above}{sc}{1.0em}{180}{$f_7$}{f7}
		
			\draw[arc] (f1) -- (ca);
			\draw[arc] (f2) -- (cb);
			\draw[arc] (f3) -- (cc);
			\draw (ca) -- (f4);
			\draw (cb) -- (f4);
			\draw (cc) -- (f4);
			\draw[arc] (f4) -- (rev);
			\draw (rev) -- (f5);
			\draw (rev) -- (f6);
			\draw (rev) -- (f7);
			\draw (ca) -- (f5);
			\draw (cb) edge[bend left = 75] (f6);
			\draw (cc) -- (f7);
			\draw[arc] (f5) -- (sa);
			\draw[arc] (f6) -- (sb);
			\draw[arc] (f7) -- (sc);
		\end{tikzpicture}
		\caption{}
		\label{fig:pcfg_example_cfg_graph}
	\end{subfigure}
	\begin{subfigure}[t]{0.49\linewidth}
		\centering
		\begin{tikzpicture}[rv/.append style={minimum height=2.2em, minimum width=5.2em}]
			\node[rv, draw] (ca) {$ComA$};
			\node[rv, draw, right = 0.5em of ca] (cb) {$ComB$};
			\node[rv, draw, right = 0.5em of cb] (cc) {$ComC$};
			\node[rv, draw, below = 7.5em of ca] (sa) {$SalA$};
			\node[rv, draw, below = 7.5em of cb] (sb) {$SalB$};
			\node[rv, draw, below = 7.5em of cc] (sc) {$SalC$};
			\node[rv, draw] (rev) at ($(cb)!0.5!(sb)$) {$Rev$};
		
			\draw[arc] (ca) -- (rev);
			\draw[arc] (cb) -- (rev);
			\draw[arc] (cc) -- (rev);
			\draw[arc] (rev) -- (sa);
			\draw[arc] (rev) -- (sb);
			\draw[arc] (rev) -- (sc);
			\draw[arc] (ca) -- (sa);
			\draw[arc] (cb) edge[bend left = 55] (sb);
			\draw[arc] (cc) -- (sc);
		\end{tikzpicture}
		\caption{}
		\label{fig:pcfg_example_cbn_graph}
	\end{subfigure}
	\caption{(a) A \ac{cfg} modelling the interplay between the competences and salaries of three employees $Alice$, $Bob$, and $Charlie$ and the revenue of a company, and (b) the underlying causal graph. We omit the potential tables of the factors for brevity.}
	\label{fig:pcfg_example_cfg_bn_graph}
\end{figure}
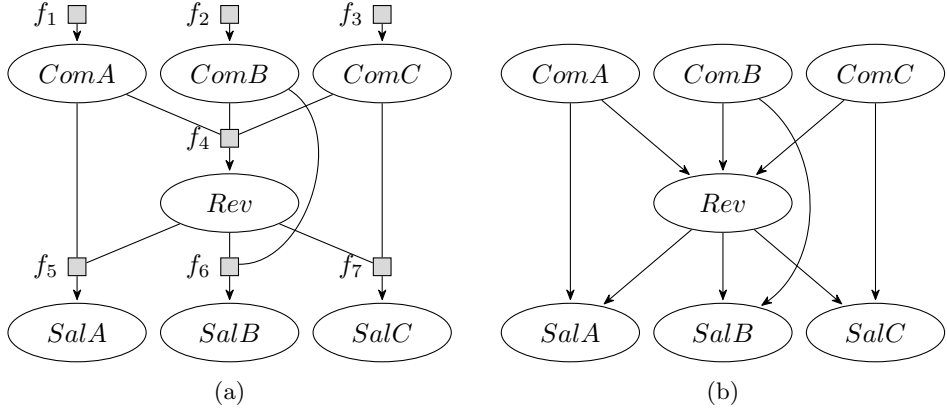

If the context is clear, we may omit the subscript $M$ in $P_M$ and simply write $P$ instead.
From the definition of the full joint probability distribution $P_M$ in \cref{eq:pcfg_cfg_joint_distribution}, it becomes clear that the full joint probability distribution $P_M$ encoded by a \ac{cfg} $M$ is independent of the edge directions in $M$.
In other words, changing the edge directions in $M$ does not affect the probability distribution $P_M$ encoded by $M$.
However, the edge directions impact the effect of an intervention and also the conditional independence statements induced by $M$, which are implied by separation in $M$.
The separation criteria in a \ac{cfg} differ from the separation criteria in a \ac{fg} as the direction of edges influences whether paths are blocked or not.
Before we define the separation criteria in a \ac{cfg}, we introduce the following notations for a \ac{cfg} $M = (\boldsymbol R \cup \boldsymbol F, \boldsymbol E, \boldsymbol \Phi)$:
\begin{itemize}
	\item $\Pa_{\boldsymbol R}(M, R) = \{ R' \in \boldsymbol R \mid \exists f \in \boldsymbol F \colon \{R', f\} \in \boldsymbol E \land (f, R) \in \boldsymbol E \}$ denotes the set of parent \acp{rv} of a \ac{rv} $R \in \boldsymbol R$ in $M$,
	\item $\Pa(M, f) = \{ R \in \boldsymbol R \mid \{R, f\} \in \boldsymbol E \}$ denotes the set of parent \acp{rv} of a factor node $f \in \boldsymbol F$ in $M$,
	\item $\Ch_{\boldsymbol R}(M, R) = \{ R' \in \boldsymbol R \mid \exists f \in \boldsymbol F \colon \{R, f\}\in \boldsymbol E \land (f, R') \in \boldsymbol E \}$ denotes the singleton set of child \acp{rv} of a \ac{rv} $R \in \boldsymbol R$ in $M$,
	\item $\Ch(M, f) = \{ R \in \boldsymbol R \mid (f, R) \in \boldsymbol E \}$ denotes the singleton set of child \acp{rv} of a factor node $f \in \boldsymbol F$ in $M$,
	\item $\De_{\boldsymbol R}(M, R) = \{ R' \in \boldsymbol R \mid \exists f_1, \allowbreak \ldots, \allowbreak f_k \in \boldsymbol F, \allowbreak R_1, \allowbreak \ldots, \allowbreak R_{k - 1} \in \boldsymbol R \colon \{R, \allowbreak f_1\}, \allowbreak (f_1, \allowbreak R_1), \allowbreak \ldots, \allowbreak \{R_{k - 1}, \allowbreak f_k\}, \allowbreak (f_k, \allowbreak R') \in \boldsymbol E \}$ denotes the set of descendant \acp{rv} of a \ac{rv} $R \in \boldsymbol R$ in $M$ (i.e., \acp{rv} that can be reached from $R$ via a directed path), and
	\item $\De(M, f) = \{ R' \in \boldsymbol R \mid \exists f_1, \allowbreak \ldots, \allowbreak f_k \in \boldsymbol F, \allowbreak R_1, \allowbreak \ldots, \allowbreak R_k \in \boldsymbol R \colon (f, \allowbreak R_1), \allowbreak \{R_1, \allowbreak f_1\}, \allowbreak \ldots, \allowbreak (f_k, \allowbreak R') \in \boldsymbol E \}$ denotes the set of descendant \acp{rv} of a factor node $f \in \boldsymbol F$ in $M$ (i.e., \acp{rv} that can be reached from $f$ via a directed path).
\end{itemize}
The subscript $\boldsymbol R$ indicates that the sets are defined with respect to the \acp{rv} in $\boldsymbol R$, that is, the sets contain neighbouring \acp{rv} (connected via a factor node) instead of directly connected factor nodes.
As factor nodes are directly connected to \acp{rv}, we omit the subscript $\boldsymbol R$ for the sets defined with respect to factor nodes.
Moreover, since every factor node has exactly one outgoing directed edge, it holds that $\abs{\Ch(M, f)} = 1$ for every factor node $f \in \boldsymbol F$.
For the ease of reading, we may write $f \to R$ (or $R \gets f$) to represent a directed edge $(f, R) \in \boldsymbol E$ and $R - f$ (or $f - R$) to represent an undirected edge $\{R, f\} \in \boldsymbol E$.
We next define separation in \acp{cfg} based on the definition given by \citet{Frey2003a} for directed \acp{fg}.

\begin{definition}[Separation in Causal Factor Graphs] \label{def:pcfg_cfg_separation}
	Let $M = (\boldsymbol R \cup \boldsymbol F, \boldsymbol E, \boldsymbol \Phi)$ denote a \ac{cfg} and let $\boldsymbol R_i \subseteq \boldsymbol R$, $\boldsymbol R_j \subseteq \boldsymbol R$, and $\boldsymbol S \subseteq \boldsymbol R$ denote pairwise disjoint sets of \acp{rv}.
	A path from a \ac{rv} to another \ac{rv} in $M$ is a connected sequence of edges and is not restricted to follow the directions of the edges.
	Thus, it is also possible for a path to pass from a parenting \ac{rv} of a factor to another parenting \ac{rv} of the same factor.
	A path is blocked by $\boldsymbol S$ if
	\begin{enumerate}
		\item \label{item:pcfg_cfg_sep_i} the path contains the pattern $f_1 \to S \gets f_2$, where $f_1, f_2 \in \boldsymbol F$, such that neither $S$ nor any of its descendants are in $\boldsymbol S$, or
		\item \label{item:pcfg_cfg_sep_ii} the path passes from $f_1$ through $S$ to $f_2$, where $f_1, f_2 \in \boldsymbol F$, such that it does not contain the pattern $f_1 \to S \gets f_2$ and $S$ is in $\boldsymbol S$, or
		\item \label{item:pcfg_cfg_sep_iii} the path passes from a parent of a factor node $f \in \boldsymbol F$ to another parent of $f$, and neither the child of $f$ nor any of its descendants are in $\boldsymbol S$.
	\end{enumerate}
	$M$ implies $(\boldsymbol R_i \upmodels \boldsymbol R_j \mid \boldsymbol S)$ if $\boldsymbol S$ \emph{separates} $\boldsymbol R_i$ and $\boldsymbol R_j$ in $M$, that is, if $\boldsymbol S$ blocks all paths from a \ac{rv} in $\boldsymbol R_i$ to a \ac{rv} in $\boldsymbol R_j$.
\end{definition}

The separation criteria for \acp{cfg} given in \cref{def:pcfg_cfg_separation} directly correspond to the rules of $d$-separation in \acp{bn} introduced by \citet{Pearl1986a}.

\begin{example}[Separation] \label{ex:pcfg_cfg_separation}
	Consider again the \ac{cfg} $M$ depicted in \cref{fig:pcfg_example_cfg_graph}.
	For instance, it holds that $ComA \upmodels ComB$ (i.e., $\boldsymbol S = \emptyset$) as all paths from $ComA$ to $ComB$ are blocked by the condition given in \cref{item:pcfg_cfg_sep_iii} from \cref{def:pcfg_cfg_separation}.
	However, it holds that $ComA \nupmodels ComB \mid \{ Rev \}$ as, for instance, the path $ComA - f_4 - ComB$ is not blocked anymore as soon as $Rev \in \boldsymbol S$.
\end{example}

When using a \ac{cfg} $M$ to encode a probability distribution, it is crucial that the conditional independence statements induced by $M$ actually hold in the probability distribution.
The \emph{global Markov property} ensures that separation criteria in a \ac{cfg} are compliant with the conditional independence statements in a probability distribution.

\begin{definition}[Global Markov Property~\citep{Lauritzen1996a}] \label{def:prelim_global_markov_property}
	A probability distribution $P$ satisfies the \emph{global Markov property} for a \ac{cfg} $M$ if and only if for all disjoint sets of variables $\boldsymbol R_i$, $\boldsymbol R_j$, and $\boldsymbol S$ it holds that if $\boldsymbol R_i$ is separated from $\boldsymbol R_j$ given $\boldsymbol S$ in $M$, then $\boldsymbol R_i$ is conditionally independent from $\boldsymbol R_j$ given $\boldsymbol S$ in $P$.
\end{definition}

In other words, every conditional independence statement induced by the graph structure of a \ac{cfg} $M$ also holds in the probability distribution $P$ if the global Markov property is satisfied.
Thus, when encoding a probability distribution $P$ using a \ac{cfg} $M$, $M$ cannot be chosen arbitrarily but instead must be chosen such that $P$ satisfies the global Markov property with respect to $M$.
In this article, we therefore assume that all distributions $P$ satisfy the global Markov property for the corresponding \ac{cfg} that is used to encode $P$.
Furthermore, we additionally require that $M$ and $P$ fulfil the causal Markov property~\citep{Spirtes2000a}, stating that in $M$, every \ac{rv} $R \in \boldsymbol R$ is independent of all \acp{rv} that are neither effects nor direct causes of $R$ given the set of $R$'s direct causes.
By assuming that the causal Markov property holds, we ensure that every directed edge in $M$ accurately represents a causal relationship between the involved \acp{rv} (i.e., edge directions are actually causal).

\begin{definition}[Causal Markov Property~\citep{Spirtes2000a}] \label{def:pcfg_causal_markov_property}
	Let $M = (\boldsymbol R \cup \boldsymbol F, \boldsymbol E, \boldsymbol \Phi)$ be a \ac{cfg} and let $P$ be a probability distribution over the \acp{rv} in $\boldsymbol R$ generated by the underlying causal structure of $M$.
	$M$ and $P$ satisfy the \emph{causal Markov property} if and only if for every \ac{rv} $R \in \boldsymbol R$, $R$ is independent of its non-descendants given its parents, i.e., $R \upmodels \boldsymbol R \setminus (\De_{\boldsymbol R}(M, R) \cup \Pa_{\boldsymbol R}(M, R)) \mid \Pa_{\boldsymbol R}(M, R)$.
\end{definition}

In case a \ac{cfg} $M$ and a probability distribution $P$ satisfy the causal Markov property, $P$ also satisfies the global Markov property with respect to $M$, that is, the global Markov property is implied by the causal Markov property.
However, it is generally possible for a probability distribution $P$ to satisfy the global Markov property for a directed (non-causal) \ac{fg} $M$ without $M$ and $P$ satisfying the causal Markov property if $M$ encodes the conditional independence statements in $P$ but does not represent the true underlying causal relationships of $P$ (i.e., the directed edges in $M$ are not causal).
In case the causal Markov property is satisfied, direct causes of a \ac{rv} $R$ are given by the parents of $R$ and the effects of $R$ are given by the children of $R$ and hence, $P$ factorises as
\begin{align} \label{eq:pcfg_causal_markov_factorisation}
	P(R_1 = r_1, \ldots, R_n = r_n) = \prod_{i = 1}^{n} P(R_i = r_i \mid \Pa_{\boldsymbol R}(M, R_i) = \pa_{\boldsymbol R}(M, R_i)),
\end{align}
where $\pa_{\boldsymbol R}(M, R_i)$ denotes a projection of the assignment $(r_1, \ldots, r_n)$ to the parents of $R_i$.
Moreover, we assume causal sufficiency~\citep{Spirtes2000a} in this article.
The causal sufficiency assumption states that all common causes of \acp{rv} that are included in a model are also included in the model.

\begin{definition}[Causal Sufficiency] \label{def:pcfg_causal_sufficiency}
	A set $\boldsymbol R$ of \acp{rv} is \emph{causally sufficient} if and only if every common cause of any two \acp{rv} in $\boldsymbol R$ is also in $\boldsymbol R$.
\end{definition}

To summarise, whenever we deal with a \ac{cfg} (or any of the causal models introduced in the upcoming sections) $M$, we make the following assumptions:
\begin{enumerate}
	\item $M$ is acyclic, i.e., $M$ contains no directed cycles (\cref{def:pcfg_cfg}),
	\item $M$ and the probability distribution $P$ encoded by $M$ satisfy the causal Markov property (\cref{def:pcfg_causal_markov_property}), and
	\item the set of \acp{rv} $\boldsymbol R$ in $M$ is causally sufficient (\cref{def:pcfg_causal_sufficiency}).
\end{enumerate}
We next introduce a formal notion of an intervention in a \ac{cfg}.
An intervention $do(R = r)$ sets the value of a \ac{rv} $R$ to a fixed value $r \in \range{R}$ and removes all incoming influences on $R$~\citep{Pearl2016a}.
In the following, we use the notation $do(R_1 = r_1, \ldots, R_k = r_k)$ to denote a joint intervention on the \acp{rv} $R_1, \ldots, R_k$, which is an abbreviation of $do(R_1 = r_1), \ldots, do(R_k = r_k)$.
When performing an intervention, the underlying probability distribution changes.
The next definition formalises the effect of an intervention on the probability distribution encoded by a \ac{cfg}.

\begin{definition}[Interventional Distribution] \label{def:pcfg_interventional_distribution}
	Let $M = (\boldsymbol R \cup \boldsymbol F, \boldsymbol E, \boldsymbol \Phi)$ be a \ac{cfg} with $\boldsymbol R = \{ R_1, \allowbreak \ldots, \allowbreak R_n \}$.
	An intervention $do(R'_1 = r'_1, \ldots, R'_k = r'_k)$ on the \acp{rv} $R'_1, \ldots, R'_k \in \boldsymbol R$ changes the probability distribution $P_M$ encoded by $M$ such that
	\begin{align*}
		P_M&(R_1 = r_1, \ldots, R_n = r_n \mid do(R'_1 = r'_1, \ldots, R'_k = r'_k)) \\
		&=
		\begin{cases}
			\prod\limits_{R_i \in \{ R_1, \ldots, R_n \} \setminus \{ R'_1, \ldots, R'_k \}} P(r_i \mid \pa_{\boldsymbol R}(M, R_i)) & \text{if } \forall j \in \{1, \ldots, k\}\colon r_j = r'_j \\
			0 & \text{otherwise},
		\end{cases}
	\end{align*}
	where $\pa_{\boldsymbol R}(M, R_i)$ denotes a projection of the assignment $(r_1, \ldots, r_n)$ to the parents $\Pa_{\boldsymbol R}(M, R_i)$ of $R_i$.
	$P_M(R_1 = r_1, \allowbreak \ldots, \allowbreak R_n = r_n \mid do(R'_1 = r'_1, \allowbreak \ldots, \allowbreak R'_k = r'_k))$ is called the \emph{interventional distribution} of $P_M$ under the intervention $do(R'_1 = r'_1, \ldots, R'_k = r'_k)$.
\end{definition}

The interventional distribution is used to compute the effect of an intervention, that is, to answer an interventional query, which we define as follows.

\begin{definition}[Interventional Query] \label{def:pcfg_interventional_query}
	An \emph{interventional query} $P(Q \mid do(R_1 = r_1, \ldots, R_k = r_k))$ consists of a query term $Q$ (also called query variable) and a set of interventions $\boldsymbol I = \{ do(R_1 = r_1), \ldots, do(R_k = r_k) \}$ where $Q$ and $R_1, \ldots, R_k$ are disjoint \acp{rv}.
	We also refer to the variables $R_1, \ldots, R_k$ in $\boldsymbol I$ as intervention variables.
	To query a specific probability instead of a probability distribution, the query term is an event $Q = q$.
\end{definition}

To answer an interventional query, we can directly apply \cref{def:pcfg_interventional_distribution}.
In particular, given a \ac{cfg} $M = (\boldsymbol R \cup \boldsymbol F, \boldsymbol E, \boldsymbol \Phi)$ with $\boldsymbol R = \{ R_1, \ldots, R_{\ell}, R'_1, \ldots, R'_k \}$, for an intervention $do(R'_1 = r'_1, \ldots R'_k = r'_k)$, the joint distribution over $R_1, \ldots, R_{\ell}$ is given as
\begin{align} \label{eq:pcfg_truncated_product_formula}
\begin{split}
	P(&R_1 = r_1, \ldots, R_{\ell} = r_{\ell} \mid do(R'_1 = r'_1, \ldots, R'_k = r'_k)) \\
	&= \prod_{R_i \in \{ R_1, \ldots, R_{\ell} \}} P(R_i = r_i \mid \Pa_{\boldsymbol R}(M, R_i) = \pa_{\boldsymbol R}(M, R_i)),
\end{split}
\end{align}
where $\pa_{\boldsymbol R}(M, R_i)$ is a projection of the assignment $(r_1, \ldots, r_{\ell}, r'_1, \ldots, r'_k)$ to the parents $\Pa_{\boldsymbol R}(M, R_i)$ of $R_i$.
\Cref{eq:pcfg_truncated_product_formula} is also known under the name of the truncated product formula or g-formula~\citep{Pearl2016a}.
If we are not interested in the joint distribution over \emph{all} \acp{rv} $R_1, \ldots, R_{\ell}$ that are not intervened on but instead wish to compute the distribution over a subset of $\{ R_1, \ldots, R_{\ell} \}$, we have to sum out all \acp{rv} from $R_1, \ldots, R_{\ell}$ that are not queried.

\begin{definition}[Truncated Product Formula]
	Let $M = (\boldsymbol R \cup \boldsymbol F, \boldsymbol E, \boldsymbol \Phi)$ be a \ac{cfg}.
	Further, let $\boldsymbol R = \{ Q, R_1, \ldots, R_{\ell}, R'_1, \ldots, R'_k \}$.
	The result of an interventional query $P(Q \mid do(R'_1 = r'_1, \ldots, R'_k = r'_k))$ is then given by
	\begin{align}
	\begin{split}
		P(&Q \mid do(R'_1 = r'_1, \ldots R'_k = r'_k)) \\
		&= \sum_{r_1 \in \range{R_1}} \ldots \sum_{r_{\ell} \in \range{R_{\ell}}} P(Q \mid \pa_{\boldsymbol R}(M, Q)) \cdot \prod_{R_i \in \{ R_1, \ldots, R_{\ell} \}} P(r_i \mid \pa_{\boldsymbol R}(M, R_i)),
	\end{split}
	\end{align}
	where $r_i$ is again shorthand for $R_i = r_i$ (analogously for $\pa_{\boldsymbol R}$) and $\pa_{\boldsymbol R}(M, Q)$ as well as $\pa_{\boldsymbol R}(M, R_i)$ denote projections of the assignment $(q, r_1, \ldots, r_{\ell}, r'_1, \ldots, r'_k)$ to the parents of $Q$ and $R_i$, respectively.
\end{definition}

So far, we introduced \acp{cfg} as propositional causal models to compute the effect of interventions.
Next, we combine lifted representations with causal knowledge to allow for lifted causal inference.

\section{Parametric Causal Factor Graphs} \label{sec:pcfg_pcfg}
A \ac{pcfg} combines a \ac{cfg} and relational logic~\citep{Genesereth2017a} (that is, first-order logic with known universes).
By incorporating relational logic, a \ac{pcfg} allows to encode that certain properties hold for all objects in a group (i.e., set) of objects.
In a \ac{pcfg}, \acp{prv} and \acp{pf} represent sets of \acp{rv} and factors, respectively.
More specifically, a \ac{prv} is parameterised by \acp{lv}, each having a domain consisting of constants, to represent a set of \acp{rv}.
Replacing the \acp{lv} with constants from their respective domains, called \emph{grounding}, results in classical \acp{rv} again.
To restrict \acp{lv} to specific constants from their respective domains, \acp{prv} are provided with constraints.
We first define \acp{prv} and their components and afterwards define \acp{pf}, before we introduce \acp{pcfg} and their semantics.

\subsection{Definitions}
The upcoming definitions are based on the definitions given by \citet{Braun2020a} but have been adapted and extended.
\begin{definition}[Parameterised Random Variable]
	Let $\boldsymbol R$ be a set of \ac{rv} names, $\boldsymbol L$ a set of \ac{lv} names, and $\boldsymbol D$ a set of constants.
	All sets are finite.
	Each \ac{lv} $L \in \boldsymbol L$ has a domain $\domain{L} \subseteq \boldsymbol D$.
	A \emph{constraint} $C = (\mathcal L, \boldsymbol C_{\mathcal L})$ is a tuple of a sequence of \acp{lv} $\mathcal L = (L_1, \ldots, L_n)$ and a set $\boldsymbol C_{\mathcal L} \subseteq \times_{i = 1}^{n} \domain{L_i}$.
	The symbol $\top$ for $C$ marks that no restrictions apply, i.e., $\boldsymbol C_{\mathcal L} = \times_{i = 1}^{n} \domain{L_i}$.
	A \emph{substitution} $\boldsymbol \sigma = \{ L_i \mapsto t_i \}_{i = 1}^{n}$ replaces every occurrence of \ac{lv} $L_i$ with term $t_i \in \domain{L_i}$ (also called grounding).
	A \emph{\ac{prv}} $R(L_1, \ldots, L_n)$, $n \geq 0$, is a syntactical construct of a \ac{rv} name $R \in \boldsymbol R$ possibly combined with \acp{lv} $L_1, \ldots, L_n \in \boldsymbol L$ to represent a set of \acp{rv}.
	If $n = 0$, the \ac{prv} is parameterless and forms a propositional \ac{rv}.
	A \ac{prv} $A$ (or \ac{lv} $L$) under constraint $C$ is given by $A_{| C}$ ($L_{| C}$, respectively).
	We may omit $| \top$ in $A_{| \top}$ or $L_{| \top}$.
	The term $\range{A}$ denotes the possible values of a \ac{prv} $A$.
	An \emph{event} $A = a$ denotes the occurrence of \ac{prv} $A$ with range value $a \in \range{A}$ and a set of events $\boldsymbol \Xi = \{ A_1 = a_1, \ldots, A_k = a_k \}$ is called \emph{evidence}.
\end{definition}

We further denote by $\lv(Y)$ the \acp{lv} occurring in $Y$, where $Y$ may be a \ac{prv} or a constraint.
The set of all instances of $Y$ (a \ac{lv} or \ac{prv}) with respect to given constraints is denoted by $\gr(Y)$.
An instance (also called grounding) of $Y$ is the result of substituting the \acp{lv} in $Y$ with constants from the specified constraints.
For a set of elements $\boldsymbol Y$ (e.g., \acp{lv} or \acp{prv}), we define $\lv(\boldsymbol Y) = \bigcup_{Y \in \boldsymbol Y} \lv(Y)$ and $\gr(\boldsymbol Y) = \bigcup_{Y \in \boldsymbol Y} \gr(Y)$.
The next example introduces \acp{prv} for our running example.

\begin{example}[Parameterised Random Variable]
	Consider $\boldsymbol{R} = \{ Com, Rev, Sal \}$ for competence, revenue, and salary, respectively, $\boldsymbol{L} = \{ E \}$ with $\domain{E} = \{ Alice, \allowbreak Bob, \allowbreak Charlie \}$ (employees), and $\boldsymbol D = \{ Alice, \allowbreak Bob, \allowbreak Charlie \}$.
	Combining $Com$ and $Sal$ with the \ac{lv} $E$, we obtain the \acp{prv} $Com(E)_{| \top} = Com(E)$ and $Sal(E)_{| \top} = Sal(E)$.
	Furthermore, $Rev$ is a parameterless \ac{prv}.
	For the sake of the example, let $\range{Com(E)} = \range{Rev} = \range{Sal(E)} = \{ \low, \allowbreak \high \}$.
	Applying the substitution $\boldsymbol \sigma = \{ E \mapsto Alice \}$ to $Com(E)$ results in $Com(Alice)$.
	The groundings of $Com(E)$ are given by $\gr(Com(E)) = \{ Com(Alice), \allowbreak Com(Bob), \allowbreak Com(Charlie) \}$.
	Applying the constraint $C = (E, \{ Alice \})$ to $Com(E)$ yields $Com(E)_{| (E, \{ Alice \})}$ with groundings $\gr(Com(E)_{| (E, \{ Alice \})}) = \{ Com(Alice) \}$.
\end{example}

We next define \acp{pf}, which represent sets of factors and are used to encode the probability distribution over the \acp{rv}.
A \ac{pf} describes a function, mapping argument values to positive real numbers (potentials), of which at least one is non-zero.

\begin{definition}[Parfactor]
	Let $\boldsymbol \Phi$ denote a set of function definitions, let $\mathcal A = (A_1, \ldots, A_n)$ denote a sequence of \acp{prv}, and let $(\mathcal L, \boldsymbol C_{\mathcal L})$ denote a constraint on the \acp{lv} $\mathcal L$ in $\mathcal A$.
	With $\phi \colon \times_{i = 1}^{n} \range{A_i} \mapsto \mathbb{R}_{\geq 0}$ being a function from $\boldsymbol \Phi$, a \emph{\ac{pf}} is given by $\forall \boldsymbol l \in \boldsymbol C_{\mathcal L} \colon \phi(\mathcal A)_{| (\mathcal L, \boldsymbol C_{\mathcal L})}$, where $\mathcal L$ is substituted by $\boldsymbol l$ in $\mathcal A$.
	We write $\phi(\mathcal A)_{| (\mathcal L, \boldsymbol C_{\mathcal L})}$ as a shorthand for $\forall \boldsymbol l \in \boldsymbol C_{\mathcal L} \colon \phi(\mathcal A)_{| (\mathcal L, \boldsymbol C_{\mathcal L})}$ (omitting the substitution) and we again may omit $| \top$ in $\phi(\mathcal A)_{| \top}$.
\end{definition}

For a \ac{pf} $\phi$, $\lv(\phi)$ again refers to the \acp{lv} in $\phi$ and $\gr(\phi)$ again denotes the set of instances of $\phi$.
We next introduce \acp{pf} for our running example.

\begin{example}[Parfactor]
	Take a look at $\phi_1(Com(E))_{| \top}$ with $\range{Com(E)} = \{ \low, \allowbreak \high \}$ and $\domain{E} = \{ Alice, \allowbreak Bob, \allowbreak Charlie \}$.
	For $\phi_1$, we have
	\begin{align*}
		\phi_1(Com(E))_{| \top} &= \forall e \in \domain{E} \colon \phi_1(Com(e))_{| \top}.
	\end{align*}
	It holds that $\gr(\phi_1(Com(E))_{| \top}) = \{ \phi_1(Com(Alice)), \allowbreak \phi_1(Com(Bob)), \allowbreak \phi_1(Com(Charlie)) \}$.
	In this specific example, $\phi_1(Com(E))_{| \top}$ thus represents a set of three ground factors.
\end{example}

Before we are ready to define a \ac{pcfg}, we need one more concept, namely the concept of a \ac{crv}~\citep{Milch2008a}, which allows us to compactly encode a factor where it does not matter which specific individual \acp{rv} have a certain range value but instead only the number of \acp{rv} having particular range values is of interest.
The range of a \ac{crv} is the space of histograms, i.e., a range value is a histogram indicating how many \acp{rv} have a certain value.

\begin{definition}[Counting Random Variable]
	Let $A(\mathcal{L})_{| C}$ denote a \ac{prv} under constraint $C$, where $\lv(\mathcal{L}) = \{ L \}$, i.e., either $\mathcal{L}$ consists of only $L$ or the other inputs are constants (meaning $\mathcal L$ contains at most one \ac{lv}).
	We denote a \emph{\ac{crv}} by $\#_{L}[A(\mathcal{L})_{| C}]$.
	Its range is the space of possible histograms.
	A histogram $h$ is a set of tuples $\{ (v_i, n_i) \}_{i = 1}^{m}$, $v_i \in \range{A(\mathcal{L})}$, $n_i \in \mathbb{N}$, $m = \abs{\range{A(\mathcal{L})}}$, and $\sum_i n_i = \abs{\gr(L_{| C})}$ for some constraint $C$ over $\mathcal{L}$.
	A shorthand notation is $[n_1, \ldots, n_m]$.
	Since counting binds the \ac{lv} $L$, $\lv(\#_{L}[A(\mathcal{L})]) = \mathcal{L} \setminus \{ L \}$.
\end{definition}

\begin{example}[Counting Random Variable] \label{ex:prelim_crv}
	Let $\#_{E}[Com(E)]$ be a \ac{crv}, $\range{Com(E)} = \{ \low, \allowbreak \high \}$ and $\domain{E} = \{ Alice, \allowbreak Bob, \allowbreak Charlie \}$.
	Then, there are $m = \abs{\range{Com(E)}} = 2$ possible range values and $n = \abs{\gr(E)} = 3$ groundings.
	Hence, the histograms are $[0, 3]$, $[1, 2]$, $[2, 1]$, and $[3, 0]$ (corresponding to $\{ (\high, 0), \allowbreak (\low, 3) \}$, $\{ (\high, 1), \allowbreak (\low, 2) \}$, $\{ (\high, 2), \allowbreak (\low, 1) \}$ and $\{ (\high, 3), \allowbreak (\low, 0) \}$ in set notation, respectively).
\end{example}

We have now introduced all components of a \ac{pcfg}, which we define next.

\begin{definition}[Parametric Causal Factor Graph] \label{def:pcfg_pcfg}
	A \emph{\ac{pcfg}} $M = (\boldsymbol V, \boldsymbol E, \boldsymbol \Phi)$ consists of a directed graph $(\boldsymbol V, \boldsymbol E)$ with node set $\boldsymbol V = \boldsymbol A \cup \boldsymbol G$ and edge set $\boldsymbol E \subseteq \boldsymbol A \times \boldsymbol G$.
	The set of nodes $\boldsymbol V = \boldsymbol A \cup \boldsymbol G$ is partitioned into a set of \acp{prv} $\boldsymbol A = \{ A_1, \ldots, A_n \}$ and a set of \ac{pf} names (\ac{pf} nodes) $\boldsymbol G = \{ g_1, \ldots, g_m \}$.
	For every \ac{pf} name $g_j \in \boldsymbol G$, there is a function definition (\ac{pf}) $\phi_j(\mathcal A_j)_{| C} \in \boldsymbol \Phi$ with $\mathcal A_j$ being a sequence of \acp{prv} from $\boldsymbol A$ and $C$ being a constraint on the \acp{lv} of $\mathcal A_j$ such that $\phi\colon \times_{A \in \mathcal A_j} \range{A} \mapsto \mathbb{R}_{\geq 0}$ maps range values in $\mathcal A_j$ to a positive real number (potential).
	In every function definition, at least one potential has to be non-zero and we again may omit $| \top$ in $\phi_j(\mathcal A_j)_{| \top}$.
	For each \ac{pf} name $g_j \in \boldsymbol G$ with corresponding function definition $\phi_j(\mathcal A_j)_{| C} \in \boldsymbol \Phi$, there is either an undirected edge $\{ A_i, g_j \} \in \boldsymbol E$ or a directed edge $(g_j, A_i) \in \boldsymbol E$ for every \ac{prv} $A_i \in \mathcal A_j$ (directed edges are only allowed to point from \ac{pf} nodes to \acp{prv} but not vice versa).
	We stipulate that for every \ac{pf} node $g_j \in \boldsymbol G$, there is exactly one outgoing directed edge $(g_j, A_i) \in \boldsymbol E$ among the edges incident to $g_j$.
	Each directed edge $A_i - g_j \to A_k$ from a \ac{prv} $A_i \in \boldsymbol A$ to a \ac{prv} $A_k \in \boldsymbol A$ via a \ac{pf} node $g_j \in \boldsymbol G$ corresponds to a direct causal relationship between $A_i$ and $A_k$.
	A \ac{pcfg} is an acyclic graph, that is, $\boldsymbol E$ contains no sequence of edges $\{ A_1, \allowbreak g_1 \}, \allowbreak (g_1, \allowbreak A_2), \allowbreak \ldots, \allowbreak \{ A_{k - 1}, \allowbreak g_k \}, \allowbreak (g_k, \allowbreak A_1)$ starting from an arbitrary \ac{prv} $A_1 \in \boldsymbol A$ such that the sequence ends again at $A_1$ when following the edges in the direction of the arrows.
	The semantics of $M$ is given by grounding with respect to constraints and building a full joint distribution over $\boldsymbol R = \gr(\boldsymbol A)$.
	The joint potential for an assignment $\boldsymbol R = \boldsymbol r$ is
	\begin{align}
		\psi_M(\boldsymbol R = \boldsymbol r) = \prod_{\phi_j \in \boldsymbol \Phi} \prod_{\phi_k \in \gr(\phi_j)} \phi_k(\mathcal R_k = \boldsymbol r_k),
	\end{align}
	where $\boldsymbol r_k$ is a projection of $\boldsymbol r$ to the argument list $\mathcal R_k$ of $\phi_k$.
	The normalised joint potential then yields the full joint probability distribution over $\boldsymbol R$ that is encoded by $M$, that is,
	\begin{align} \label{eq:pcfg_joint_distribution}
		P_M(\boldsymbol R = \boldsymbol r) = \frac{1}{Z} \psi_M(\boldsymbol R = \boldsymbol r),
	\end{align}
	where $Z$ is the normalisation constant, defined as
	\begin{align} \label{eq:pcfg_normalisation_constant}
		Z = \sum_{\boldsymbol r \in \times_{R \in \boldsymbol R} \range{R}} \psi_M(\boldsymbol R = \boldsymbol r).
	\end{align}
\end{definition}

The definition of a \ac{pcfg} also implies that every \ac{cfg} is a \ac{pcfg} containing only parameterless \acp{prv} (analogously, every factor is a \ac{pf} having only parameterless arguments).
Grounding a \ac{pcfg} thus yields a \ac{cfg} entailing equivalent semantics (that is, encoding the same full joint probability distribution) as the \ac{pcfg}.
As \acp{lv} abstract from individual objects, we refer to \acp{pcfg} as \emph{lifted} representations and to \acp{cfg} as \emph{propositional} representations (in the same way, we refer to parameterless \acp{prv} as propositional \acp{rv} and to \acp{pf} having only parameterless arguments as propositional factors).
In the literature, a lifted representation is sometimes also referred to as a first-order representation.
Before we take a look at an example, we introduce the following notations for a \ac{pcfg} $M = (\boldsymbol A \cup \boldsymbol G, \boldsymbol E, \boldsymbol \Phi)$:
\begin{itemize}
	\item $\Pa_{\boldsymbol A}(M, A) = \{ A' \in \boldsymbol A \mid \exists g \in \boldsymbol G \colon \{ A',  g \} \in \boldsymbol E \land (g, A) \in \boldsymbol E \}$ denotes the set of parent \acp{prv} of a \ac{prv} $A \in \boldsymbol A$ in $M$,
	\item $\Pa(M, g) = \{ A \in \boldsymbol A \mid \{ A, g \} \in \boldsymbol E \}$ denotes the set of parent \acp{prv} of a \ac{pf} node $g \in \boldsymbol G$ in $M$,
	\item $\Ch_{\boldsymbol A}(M, A) = \{ A' \in \boldsymbol A \mid \exists g \in \boldsymbol G \colon \{ A, g \} \in \boldsymbol E \land (g, A') \in \boldsymbol E \}$ denotes the singleton set of child \acp{prv} of a \ac{prv} $A \in \boldsymbol A$ in $M$,
	\item $\Ch(M, g) = \{ A \in \boldsymbol A \mid (g, A) \in \boldsymbol E \}$ denotes the singleton set of child \acp{prv} of a \ac{pf} node $g \in \boldsymbol G$ in $M$,
	\item $\De_{\boldsymbol A}(M, A) = \{ A' \in \boldsymbol A \mid \exists g_1, \allowbreak \ldots, \allowbreak g_k \in \boldsymbol G, \allowbreak A_1, \allowbreak \ldots, \allowbreak A_{k - 1} \in \boldsymbol A \colon \{ A, \allowbreak g_1 \}, \allowbreak (g_1, \allowbreak A_1), \allowbreak \ldots, \allowbreak \{ A_{k - 1}, \allowbreak g_k \}, \allowbreak (g_k, \allowbreak A') \in \boldsymbol E \}$ is the set of descendant \acp{prv} of a \ac{prv} $A \in \boldsymbol A$ in $M$, and
	\item $\De(M, g) = \{ A' \in \boldsymbol A \mid \exists g_1, \allowbreak \ldots, \allowbreak g_k \in \boldsymbol G, \allowbreak A_1, \allowbreak \ldots, \allowbreak A_k \in \boldsymbol A \colon (g, \allowbreak A_1), \allowbreak \{ A_1, \allowbreak g_1 \}, \allowbreak \ldots, \allowbreak (g_k, \allowbreak A') \in \boldsymbol E \}$ is the set of descendant \acp{prv} of a \ac{pf} node $g \in \boldsymbol G$ in $M$.
\end{itemize}
As before, the subscript $\boldsymbol A$ indicates that the sets are defined with respect to the \acp{prv} in $\boldsymbol A$, that is, the sets contain neighbouring \acp{prv} that are connected via a \ac{pf} node.
Furthermore, as the definition of a \ac{pcfg} requires every \ac{pf} node to have exactly one outgoing directed edge, it holds that $\abs{\Ch(M, g)} = 1$ for every \ac{pf} node $g \in \boldsymbol G$.
We next give an example.

\begin{example}[Parametric Causal Factor Graph] \label{ex:pcfg_example_pcfg}
	\Cref{fig:pcfg_example_pcfg} displays a \ac{pcfg} $M$ for our running example.
	$M$ contains two \acp{prv} $Com(E)$ (for the competence of employees) and $Sal(E)$ (for the salary of employees), as well as a propositional \ac{rv} $Rev$ (for the revenue of the company).
	The ranges of the \acp{prv} are $\range{Com(E)} = \range{Sal(E)} = \range{Rev} = \{ \low, \allowbreak \high \}$ and the \ac{lv} $E$ (representing employees) has the domain $\domain{E} = \{ Alice, Bob, Charlie \}$.
	There are three \ac{pf} nodes $g_1$, $g_2$, and $g_3$ with corresponding function definitions $\phi_1(Com(E))$, $\phi_2(\#_E[Com(E)], \allowbreak Rev)$, and $\phi_3(Com(E), \allowbreak Rev, \allowbreak Sal(E))$.
	We omit the potential tables of the \acp{pf} for brevity.
	As $Com(E)$ appears count-converted in $\phi_2(\#_E[Com(E)], \allowbreak Rev)$, it holds that $\lv(\phi_2) = \emptyset$ and thus, $g_2$ is not layered in \cref{fig:pcfg_example_pcfg} while $g_1$ and $g_3$ are layered (because $\lv(\phi_1) \neq \emptyset$ and $\lv(\phi_3) \neq \emptyset$).
	In set notation, $M = (\boldsymbol A \cup \boldsymbol G, \boldsymbol E, \boldsymbol \Phi)$ is given as
	\begin{align*}
		\boldsymbol R &= \{ Com, Rev, Sal \}, \\
		\boldsymbol L &= \{ E \}, \\
		\boldsymbol D &= \{ Alice, Bob, Charlie \}, \\
		\boldsymbol A &= \{ Com(E), Rev, Sal(E) \}, \\
		\boldsymbol G &= \{ g_1, g_2, g_3 \}, \\
		\boldsymbol E &= \{ (g_1, Com(E)), \{Com(E), g_2\}, \{Com(E), g_3\}, (g_2, Rev), \{Rev, g_3\}, (g_3, Sal(E)) \}, \\
		\boldsymbol \Phi &= \{ \phi_1(Com(E)), \phi_2(\#_{E}[Com(E)], \allowbreak Rev), \phi_3(Com(E), \allowbreak Rev, \allowbreak Sal(E)) \},
	\end{align*}
	where $\boldsymbol R$ is the set of \ac{rv} names, $\boldsymbol L$ is the set of \ac{lv} names, and $\boldsymbol D$ is the set of constants.
	Grounding $M$ results in the \ac{cfg} from \cref{ex:pcfg_cfg}, where $Com(Alice)$ corresponds to $ComA$, $Com(Bob)$ corresponds to $ComB$, and so on.
\end{example}

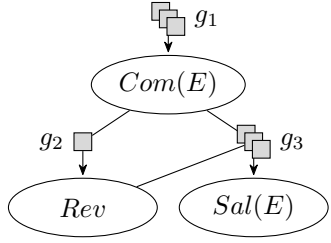
\begin{figure}[t]
	\centering
	\begin{tikzpicture}[rv/.append style={minimum height=2.2em, minimum width=5.6em}]
		\node[rv, inner sep = 1.8pt] (C) {$Com(E)$};

		\pfs{above}{C}{1.0em}{0}{$g_1$}{G1a}{G1}{G1b}
		\factor{below left}{C}{1.0em and 0.8em}{180}{$g_2$}{G2}
		\pfs{below right}{C}{1.0em and 0.8em}{0}{$g_3$}{G3a}{G3}{G3b}

		\node[rv, below = 1.0em of G2] (R) {$Rev$};
		\node[rv, inner sep = 1.8pt, below = 1.0em of G3] (S) {$Sal(E)$};

		\begin{pgfonlayer}{bg}
			\draw[arc] (G1) -- (C);
			\draw (C) -- (G2);
			\draw (C) -- (G3);
			\draw[arc] (G2) -- (R);
			\draw (R) -- (G3);
			\draw[arc] (G3) -- (S);
		\end{pgfonlayer}
	\end{tikzpicture}
	\caption{An illustration of a \ac{pcfg} for our running example. We omit the potential tables of the (par)factors for brevity.}
	\label{fig:pcfg_example_pcfg}
\end{figure}

We deliberately chose labels $ComA$, $ComB$, and $ComC$ instead of $Com(Alice)$, $Com(Bob)$, and $Com(Charlie)$ in \cref{ex:pcfg_cfg} to emphasise that there is no explicit representation of objects (here employees) in the graph structure of the propositional \ac{cfg}.
In general, node labels can be arbitrary strings of characters.
The size of the \ac{pcfg} (that is, the number of nodes and edges in the graph) remains constant even if the number of employees increases.
In the \ac{cfg} from \cref{fig:pcfg_example_cfg_graph}, however, the size of the graph increases linearly with the number of employees as every additional employee adds two \acp{rv} and two factors to the graph.
In general, there might be multiple groups of indistinguishable objects instead of having a single group including all objects, which can be represented by using constraints.

Before we continue to define the semantics of an intervention in a \ac{pcfg}, we briefly provide the separation criteria in a \ac{pcfg}, linking its graph structure to conditional independence statements in the underlying probability distribution.

\subsection{Conditional Independence in Parametric Causal Factor Graphs} \label{sec:pcfg_ci}
The separation criteria in a \ac{pcfg} are given on a ground level and hence directly correspond to the separation criteria for a \ac{cfg} given in \cref{def:pcfg_cfg_separation}.

\begin{definition}[Separation in Parametric Causal Factor Graphs] \label{def:pcfg_pcfg_separation}
	Let $M = (\boldsymbol A \cup \boldsymbol G, \boldsymbol E, \boldsymbol \Phi)$ be a \ac{pcfg}.
	Further, let $\boldsymbol R = \gr(\boldsymbol A)$ and let $\boldsymbol R_i \subseteq \boldsymbol R$, $\boldsymbol R_j \subseteq \boldsymbol R$, and $\boldsymbol S \subseteq \boldsymbol R$ be pairwise disjoint sets of ground \acp{rv}.
	We say that $\boldsymbol S$ \emph{separates} $\boldsymbol R_i$ and $\boldsymbol R_j$ in $M$ if $\boldsymbol S$ blocks all paths from any \ac{rv} in $\boldsymbol R_i$ to any \ac{rv} in $\boldsymbol R_j$ in $\gr(M)$.
	$M$ implies the conditional independence statement $(\boldsymbol R_i \upmodels \boldsymbol R_j \mid \boldsymbol S)$ if $\boldsymbol S$ separates $\boldsymbol R_i$ and $\boldsymbol R_j$ in $M$.
\end{definition}

While the separation criteria in a \ac{pcfg} are defined on a ground level, it is nevertheless possible to check for separation on a lifted level without grounding the entire \ac{pcfg}.
The Bayes-Ball algorithm~\citep{Shachter1998a} allows to efficiently check for induced conditional independence statements in a propositional \ac{bn} (and hence can also be applied to a \ac{cfg} by taking the underlying causal graph structure into account).
\citet{Meert2010a} extend the Bayes-Ball algorithm to the lifted setting, thereby allowing to run it directly on a lifted representation.

It is also possible to check for implied conditional independence statements involving \acp{prv} instead of ground \acp{rv} in a highly efficient manner on a lifted level.
In a \ac{pcfg}, every \ac{prv} $A(\mathcal L)_{| C}$ is represented by a variable node and thus, checking for conditional independence statements that involve $A$ can be done by looking at this specific variable node instead of taking into account all groundings of $A$.
In contrast, in a propositional setting (i.e., in a ground model), each ground \ac{rv} in $\gr(A)$ is an individual node in the graph and hence must be looked at individually.
For instance, to check whether $Com(E) \upmodels Sal(E) \mid Rev$ is implied by the \ac{pcfg} depicted in \cref{fig:pcfg_example_pcfg}, only three variable nodes are of relevance whereas $2 \cdot \abs{\domain{E}} + 1$ nodes are of relevance in the corresponding ground model shown in \cref{fig:pcfg_example_cfg_graph}.
Here, the conditional independence statement $Com(E) \upmodels Sal(E) \mid Rev$ (which is not implied by the \ac{pcfg} from \cref{fig:pcfg_example_pcfg}) is a shorthand to refer to a set of conditional independence statements resulting from substituting $E$ with every $e \in \domain{E}$, i.e., $\{ Com(e) \upmodels Sal(e) \mid Rev \}_{e \in \domain{E}}$.

For now, we additionally assume that every \ac{prv} $A_k$ in a \ac{pcfg} has exactly one parent \ac{pf} node $g_j$ such that each corresponding ground function definition $\phi_j(R_1, \allowbreak \ldots, \allowbreak R_k) \in \gr(\phi_j(A_1, \allowbreak \ldots, \allowbreak A_k))$ represents a conditional probability distribution $P(R_k \mid R_1, \allowbreak \ldots, \allowbreak R_{k - 1})$.
In other words, we assume that the ground \ac{cfg} represented by a given \ac{pcfg} directly corresponds to a \ac{cbn}.
This assumption results in a more convenient and often more efficient computation when answering interventional queries in a \ac{pcfg} but is not necessary to answer such queries in general.
Later on, we relax this assumption and show how interventional queries can still be answered on a lifted level.
In the next subsection, we apply the notion of an intervention to \acp{pcfg}.

\subsection{Interventions in Parametric Causal Factor Graphs}
An intervention in a \ac{pcfg} is defined analogously to an intervention in a \ac{cfg}.
The interventional distribution, in particular, is defined on a ground level again.

\begin{definition}[Interventional Distribution in a Parametric Causal Factor Graph] \label{def:pcfg_interventional_distribution_pcfg}
	Let $M = (\boldsymbol A \cup \boldsymbol G, \boldsymbol E, \boldsymbol \Phi)$ be a \ac{pcfg} with $\boldsymbol R = \gr(\boldsymbol A) = \{ R_1, \allowbreak \ldots, \allowbreak R_n \}$.
	Further, let $do(R'_1 = r'_1, \ldots, R'_k = r'_k)$ be an intervention on the \acp{rv} $R'_1, \ldots, R'_k \in \gr(\boldsymbol A)$.
	The interventional distribution under the intervention $do(R'_1 = r'_1, \ldots, R'_k = r'_k)$ is given by
	\begin{align*}
		P_M&(R_1 = r_1, \ldots, R_n = r_n \mid do(R'_1 = r'_1, \ldots, R'_k = r'_k)) \\
		&=
		\begin{cases}
			\prod\limits_{R_i \in \{ R_1, \ldots, R_n \} \setminus \{ R'_1, \ldots, R'_k \}} P(r_i \mid \pa_{\boldsymbol R}(\gr(M), R_i)) & \text{if } \forall j \in \{1, \ldots, k\}\colon r_j = r'_j \\
			0 & \text{otherwise},
		\end{cases}
	\end{align*}
	where $\pa_{\boldsymbol R}(\gr(M), R_i)$ denotes a projection of the assignment $(r_1, \ldots, r_n)$ to the parents $\Pa_{\boldsymbol R}(\gr(M), R_i)$ of $R_i$ in the ground model $\gr(M)$.
\end{definition}

Furthermore, we allow for interventions on \acp{prv}.
An intervention $do(A(\mathcal L)_{| C} = a)$ on a \ac{prv} $A$, where $a \in \range{A}$, can be seen as a joint intervention on all ground \acp{rv} in $\gr(A_{| C})$.
In other words, $do(A(\mathcal L)_{| C} = a)$ is equivalent to $do(R_1 = a, \ldots, R_k = a)$, where $\gr(A_{| C}) = \{ R_1, \ldots, R_k \}$.
From now on, we therefore also allow for interventional queries of the form $P(Q \mid do(A_1 = a_1, \ldots, A_k = a_k))$, where $A_1, \ldots, A_k$ are \acp{prv}.
Since any interventional query involving \acp{prv} can be reduced to an interventional query containing only parameterless \acp{rv}, we continue to work with our original definition of an interventional query (\cref{def:pcfg_interventional_query}).
To answer an interventional query in a \ac{pcfg}, we can again apply the truncated product formula.

\begin{definition}[Truncated Product Formula in a Parametric Causal Factor Graph] \label{def:pcfg_truncated_product_formula_pcfg}
	Let $M = (\boldsymbol A \cup \boldsymbol G, \boldsymbol E, \boldsymbol \Phi)$ be a \ac{pcfg} and let $\boldsymbol R = \gr(\boldsymbol A) = \{ Q, R_1, \ldots, R_{\ell}, R'_1, \ldots, R'_k \}$.
	The result of an interventional query $P(Q \mid do(R'_1 = r'_1, \ldots, R'_k = r'_k))$ is then given by
	\begin{align} \label{eq:pcfg_truncated_product_formula_pcfg}
	\begin{split}
		P(Q \mid do(R'_1 = r'_1, \ldots, R'_k = r'_k))
		= &\sum_{r_1 \in \range{R_1}} \ldots \sum_{r_{\ell} \in \range{R_{\ell}}} P(Q \mid \pa_{\boldsymbol R}(\gr(M), Q)) \\
		\cdot &\prod_{R_i \in \{ R_1, \ldots, R_{\ell} \}} P(r_i \mid \pa_{\boldsymbol R}(\gr(M), R_i)),
	\end{split}
	\end{align}
	where $\pa_{\boldsymbol R}(\gr(M), Q)$ and $\pa_{\boldsymbol R}(\gr(M), R_i)$ denote projections of the assignment $(q, \allowbreak r_1, \allowbreak \ldots, \allowbreak r_{\ell}, \allowbreak r'_1, \allowbreak \ldots, \allowbreak r'_k)$ to the parents of $Q$ and $R_i$ in the ground model $\gr(M)$, respectively.
\end{definition}

Even though both the interventional distribution and the truncated product formula are defined on a ground level, an interventional query can be answered without grounding the entire \ac{pcfg}.
In particular, \cref{eq:pcfg_truncated_product_formula_pcfg} gives us a formula that consists of a set of probabilistic queries, which can be answered on a lifted level.
Under the assumption of having a direct correspondence of \acp{pf} to conditional probability distributions, we can further simplify query answering in a \ac{pcfg} $M$.

\begin{proposition} \label{prop:pcfg_lci_intervention}
	Let $M = (\boldsymbol A \cup \boldsymbol G, \boldsymbol E, \boldsymbol \Phi)$ denote a \ac{pcfg} with each $\phi_j(R_{j_1}, \allowbreak \ldots, \allowbreak R_{j_z}) \in \gr(\boldsymbol \Phi)$ representing a conditional probability distribution $P(R_{j_z} \mid R_{j_1}, \allowbreak \ldots, \allowbreak R_{j_{z - 1}})$, let $\gr(\boldsymbol A) = \{ Q, \allowbreak R_1, \allowbreak \ldots, \allowbreak R_{\ell}, \allowbreak R'_1, \allowbreak \ldots, \allowbreak R'_k \}$, and let $P(Q \mid do(R'_1 = r'_1, \ldots, R'_k = r'_k))$ be an interventional query.
	Further, let $M' = (\boldsymbol A \cup \boldsymbol G, \allowbreak \boldsymbol E, \allowbreak \boldsymbol \Phi')$ be the \ac{pcfg} obtained by changing $\boldsymbol \Phi$ to $\boldsymbol \Phi'$ such that every factor $\phi_j(R_{j_1}, \allowbreak \ldots, \allowbreak R_{j_z}) \in \gr(\boldsymbol \Phi)$ that has a child $R_{j_z} = R'_z$ in $\{ R'_1, \allowbreak \ldots, \allowbreak R'_k \}$ is replaced by a factor $\phi'_j(R_{j_1}, \allowbreak \ldots, \allowbreak R_{j_z}) \in \gr(\boldsymbol \Phi')$ with
	\begin{align*}
		\phi'_j(R_{j_1} = r_{j_1}, \ldots, R_{j_z} = r_{j_z}) =
		\begin{cases}
			1 & \text{if } r_{j_z} = r'_z \\
			0 & \text{if } r_{j_z} \neq r'_z.
		\end{cases}
	\end{align*}
	All factors whose child is not in $\{ R'_1, \allowbreak \ldots, \allowbreak R'_k \}$ remain unchanged.
	The result of the interventional query $P_M(Q \mid do(R'_1 = r'_1, \ldots, R'_k = r'_k))$ in the original model $M$ is then given by the result of the probabilistic query $P_{M'}(Q \mid R'_1 = r'_1, \ldots, R'_k = r'_k)$ in the modified model $M'$.
\end{proposition}
\begin{proof}
	For each ground factor $\phi_j(R_{j_1}, \allowbreak \ldots, \allowbreak R_{j_z}) \in \gr(\boldsymbol \Phi)$, it holds that
	\begin{align} \label{eq:pcfg_lci_intervention_proof}
		\phi_j(R_{j_1} = r_{j_1}, \allowbreak \ldots, \allowbreak R_{j_z} = r_{j_z})
		= P(R_{j_z} = r_{j_z} \mid R_{j_1} = r_{j_1}, \allowbreak \ldots, \allowbreak R_{j_{z - 1}} = r_{j_{z - 1}})
	\end{align}
	for all assignments $(r_{j_1}, \ldots, r_{j_z})$.
	Entering \cref{eq:pcfg_lci_intervention_proof} into the truncated product formula (\cref{eq:pcfg_truncated_product_formula_pcfg}) leaves us with
	\begin{align*}
		P_M(Q \mid do(R'_1 = r'_1, \ldots, R'_k = r'_k))
		= \sum_{r_1 \in \range{R_1}} \ldots \sum_{r_{\ell} \in \range{R_{\ell}}} \prod_{\phi_j \in \hat{\gr}(\boldsymbol \Phi)} \phi_j(\mathcal R_j = \boldsymbol r_j),
	\end{align*}
	where $\hat{\gr}(\boldsymbol \Phi) = \{ \phi_j(R_{j_1}, \ldots, R_{j_z}) \in \gr(\boldsymbol \Phi) \mid R_{j_z} \notin \{ R'_1, \allowbreak \ldots, \allowbreak R'_k \} \}$ denotes the set of groundings of $\boldsymbol \Phi$ whose child is not in $\{ R'_1, \allowbreak \ldots, \allowbreak R'_k \}$ and $\boldsymbol r_j$ denotes a projection of the assignment $(q, \allowbreak r_1, \allowbreak \ldots, \allowbreak r_{\ell}, \allowbreak r'_1, \allowbreak \ldots, \allowbreak r'_k)$ to the argument list $\mathcal R_j$ of $\phi_j$.
	Now, consider the modified model $M'$ and the query $P(Q \mid R'_1 = r'_1, \ldots, R'_k = r'_k)$.
	By entering \cref{eq:pcfg_lci_intervention_proof} into the definition of the full joint probability distribution $P_{M'}$ encoded by $M'$ (\cref{eq:pcfg_joint_distribution}), we end up with
	\begin{align*}
		P_{M'}(Q \mid R'_1 = r'_1, \ldots, R'_k = r'_k)
		= \frac{1}{Z'} \sum_{r_1 \in \range{R_1}} \ldots \sum_{r_{\ell} \in \range{R_{\ell}}} \prod_{\phi'_j \in \gr(\boldsymbol \Phi')} \phi'_j(\mathcal R_j = \boldsymbol r_j).
	\end{align*}
	Furthermore, as every factor $\phi'_j$ whose child is not in $\{ R'_1, \allowbreak \ldots, \allowbreak R'_k \}$ is left unchanged (i.e., $\phi'_j(R_{j_1}, \ldots, R_{j_z}) = \phi_j(R_{j_1}, \ldots, R_{j_z})$ if $R_{j_z} \notin \{ R'_1, \ldots, R'_k \}$), it holds that
	\begin{align*}
		\gr(\boldsymbol \Phi')
		&= \{ \phi'_j(R_{j_1}, \ldots, R_{j_z}) \in \gr(\boldsymbol \Phi') \mid R_{j_z} \notin \{ R'_1, \ldots, R'_k \} \} \\
		&\cup \{ \phi'_j(R_{j_1}, \ldots, R_{j_z}) \in \gr(\boldsymbol \Phi') \mid R_{j_z} \in \{ R'_1, \ldots, R'_k \} \} \\
		&= \hat{\gr}(\boldsymbol \Phi) \cup \{ \phi'_j(R_{j_1}, \ldots, R_{j_z}) \in \gr(\boldsymbol \Phi') \mid R_{j_z} \in \{ R'_1, \ldots, R'_k \} \}.
	\end{align*}
	Due to the modifications in $M'$, it holds that every factor $\phi'_j(R_{j_1}, \allowbreak \ldots, \allowbreak R_{j_z}) \in \{ \phi'_j(R_{j_1}, \allowbreak \ldots, \allowbreak R_{j_z}) \in \gr(\boldsymbol \Phi') \mid R_{j_z} \in \{ R'_1, \allowbreak \ldots, \allowbreak R'_k \} \}$ maps any assignment that assigns $R'_1 = r'_1, \ldots, R'_k = r'_k$ to the value one.
	Consequently, we end up with
	\begin{align*}
		P_{M'}(Q \mid R'_1 = r'_1, \ldots, R'_k = r'_k)
		&= \frac{1}{Z'} \sum_{r_1 \in \range{R_1}} \ldots \sum_{r_{\ell} \in \range{R_{\ell}}} \prod_{\phi'_j \in \gr(\boldsymbol \Phi')} \phi'_j(\mathcal R_j = \boldsymbol r_j) \\
		&= \frac{1}{Z'} \sum_{r_1 \in \range{R_1}} \ldots \sum_{r_{\ell} \in \range{R_{\ell}}} \prod_{\phi_j \in \hat{\gr}(\boldsymbol \Phi)} \phi_j(\mathcal R_j = \boldsymbol r_j).
	\end{align*}
	Thus, to conclude the proof of showing that $P_M(Q \mid do(R'_1 = r'_1, \allowbreak \ldots, \allowbreak R'_k = r'_k)) = P_{M'}(Q \mid R'_1 = r'_1, \allowbreak \ldots, \allowbreak R'_k = r'_k)$, it remains to be shown that $Z' = 1$.
	The definition of the normalisation constant $Z'$ (\cref{eq:pcfg_normalisation_constant}) is given by
	\begin{align*}
		Z' = \sum_{(q, r_1, \ldots, r_{\ell}, r'_1, \ldots, r'_k) \in \times_{R \in \gr(\boldsymbol A)} \range{R}} \prod_{\phi'_j \in \gr(\boldsymbol \Phi')} \phi'_j(\mathcal R_j = \boldsymbol r_j).
	\end{align*}
	After the modification, every factor $\phi'_j(R_{j_1}, \allowbreak \ldots, \allowbreak R_{j_z}) \in \gr(\boldsymbol \Phi')$ still represents a valid conditional probability distribution $P(R_{j_z} = r_{j_z} \mid R_{j_1} = r_{j_1}, \allowbreak \ldots, \allowbreak R_{j_{z - 1}} = r_{j_{z - 1}})$ because exactly one assignment of $R_{j_z}$ is mapped to one while all other assignments are mapped to zero, thus ensuring that the sum of all assignments is one.
	We can therefore again apply \cref{eq:pcfg_lci_intervention_proof} to the definition of the normalisation constant $Z'$ and obtain
	\begin{align*}
		Z' = \sum_{(q, r_1, \ldots, r_{\ell}, r'_1, \ldots, r'_k) \in \times_{R \in \gr(\boldsymbol A)} \range{R}} \prod_{R_i \in \gr(\boldsymbol A)} P(r_i \mid \pa_{\boldsymbol R}(\gr(M'), R_i)),
	\end{align*}
	where $r_i$ is the assigned value for $R_i$ in the assignment $(q, \allowbreak r_1, \allowbreak \ldots, \allowbreak r_{\ell}, \allowbreak r'_1, \allowbreak \ldots, \allowbreak r'_k)$ and $\pa_{\boldsymbol R}(\gr(M), R_i)$ is a projection of the assignment  $(q, \allowbreak r_1, \allowbreak \ldots, \allowbreak r_{\ell}, \allowbreak r'_1, \allowbreak \ldots, \allowbreak r'_k)$ to the parents of $R_i$ in the ground model $\gr(M')$.
	In other words, $Z'$ is a sum over all entries in the full joint probability distribution, and hence, it holds that $Z' = 1$ as all entries in a full joint probability distribution must sum up to one.
\end{proof}

An alternative way of verifying that the normalisation constant is equal to one if every factor $\phi_j(R_{j_1}, \allowbreak \ldots, \allowbreak R_{j_z})$ represents a conditional probability distribution $P(R_{j_z} \mid R_{j_1}, \allowbreak \ldots, \allowbreak R_{j_{z - 1}})$ is to make use of \cref{eq:pcfg_lci_intervention_proof} in the factorisation implied by the causal Markov property (\cref{eq:pcfg_causal_markov_factorisation}).
The resulting factorisation then is equivalent to the factorisation given in the definition of the semantics of a \ac{pcfg} (\cref{eq:pcfg_joint_distribution}) for $Z = 1$ and as both factorisations are valid, $Z$ has to be equal to one.

By modifying the original model, an interventional query $P(Q \mid do(R'_1 = r'_1, \allowbreak \ldots, \allowbreak R'_k = r'_k))$ can be answered by computing the result of \emph{a single} probabilistic query in the modified model (however, the modification of the model introduces some overhead).
Using the truncated product formula directly on the original model instead, we obtain a set of multiple probabilistic queries that have to be answered.
In both cases, a lifted inference algorithm such as \ac{lve} or the \ac{ljt} algorithm (which is specifically advantageous if a set of probabilistic queries needs to be answered as a result of the truncated product formula) can be applied to answer these queries on a lifted level.

Before we give a full algorithm to efficiently answer interventional queries in a \ac{pcfg}, we explain how the modification of the original model is done such that \cref{prop:pcfg_lci_intervention} can be applied.
In particular, as only specific ground factors that have an intervention variable as a child are changed, we have to \emph{split} \acp{pf} such that modified ground factors can be separated from the remaining ground factors.
Splitting a \ac{pf} in a \ac{pcfg} $M$ results in a modified \ac{pcfg} $M'$ entailing equivalent semantics as $M$~\citep{DeSalvoBraz2005a} such that $M'$ forms a valid model on which lifted inference algorithms (such as \ac{lve} and the \ac{ljt} algorithm) can be run.
The procedure of splitting a \ac{pf} $\phi(\mathcal A)_{| C}$ on a specific instance $A(l_1, \ldots, l_z) \in \gr(A(\mathcal L_A))$, where $A(\mathcal L_A) \in \mathcal A$ is a \ac{prv} in the argument list of $\phi(\mathcal A)_{| C}$, replaces $\phi(\mathcal A)_{| C}$ by two \acp{pf} $\phi(\mathcal A)_{| C_1}$ and $\phi(\mathcal A)_{| C_2}$.
The constraints $C_1$ and $C_2$ are chosen such that the inputs of $\phi(\mathcal A)_{| C_1}$ are restricted to all sequences under constraint $C$ that contain $A(l_1, \ldots, l_z)$ and the inputs of $\phi(\mathcal A)_{| C_2}$ are restricted to the remaining input sequences under constraint $C$.


We next present the \ac{lci} algorithm, which efficiently answers interventional queries in a \ac{pcfg} on a lifted level.
The basic idea of \ac{lci} is to split \acp{pf} based on the intervention variables such that the parent factors of intervention variables are detached from their respective groups and thus can be changed according to \cref{prop:pcfg_lci_intervention}.

\section{The Lifted Causal Inference Algorithm} \label{sec:pcfg_lci}
The \ac{lci} algorithm solves the problem of efficiently computing the effect of interventions in a \ac{pcfg}.
\Ac{lci} avoids to fully ground the \ac{pcfg} if possible to benefit from lifted inference.
For instance, consider again the \ac{pcfg} $M$ illustrated in \cref{fig:pcfg_example_pcfg} and assume we would like to compute the answer to the interventional query $P(Rev \mid do(Com(Bob) = \high))$ in $M$.
As the intervention $do(Com(Bob) = \high)$ fixes the value of $Com(Bob)$ to $\high$, we have to treat $Bob$ differently from $Alice$ and $Charlie$, whose competences remain unobserved.
In other words, not all employees are indistinguishable anymore.
Nevertheless, and this is the crucial point, we can still treat $Alice$ and $Charlie$ as indistinguishable when computing the result of the interventional query $P(Rev \mid do(Com(Bob) = \high))$.

\begin{algorithm}[t]
	\caption{Lifted Causal Inference}
	\label{alg:pcfg_lci}
	\alginput{A \ac{pcfg} $M = (\boldsymbol A \cup \boldsymbol G, \boldsymbol E, \boldsymbol \Phi)$, and an interventional query $P(Q \mid do(R'_1 = r'_1, \allowbreak \ldots, \allowbreak R'_k = r'_k))$ with $\{Q, R'_1, \ldots, R'_k\} \subseteq \gr(\boldsymbol A) = \{ Q, \allowbreak R_1, \allowbreak \ldots, \allowbreak R_{\ell}, \allowbreak R'_1, \allowbreak \ldots, \allowbreak R'_k \}$.} \\
	\algoutput{The result of the interventional query $P(Q \mid do(R'_1 = r'_1, \ldots, R'_k = r'_k))$.}
	\begin{algorithmic}[1]
		\State $M' \gets$ \acs{pcfg} obtained by splitting \acp{pf} in $M$ on each $R'_i \in \{R'_1, \ldots, R'_k\}$\; \label{line:pcfg_lci_pf_splitting}
		\ForEach{$R'_i \in \{R'_1, \ldots, R'_k\}$} \label{line:pcfg_lci_outer_loop}
			\ForEach{$\phi'_j(R_{j_1}, \ldots, R_{j_z}) \in \Pa(M', R'_i)$} \label{line:pcfg_lci_middle_loop}
				\ForEach{assignment $(r_{j_1}, \ldots, r_{j_z}) \in \range{R_{j_1}} \times \ldots \times \range{R_{j_z}}$} \label{line:pcfg_lci_inner_loop}
					\State Set $\phi'_j(r_{j_1}, \ldots, r_{j_z}) = \begin{cases} 1 & \text{if $(r_{j_1}, \ldots, r_{j_z})$ assigns $R'_i = r'_i$} \\ 0 & \text{if $(r_{j_1}, \ldots, r_{j_z})$ assigns $R'_i \neq r'_i$} \end{cases}$ \label{line:pcfg_lci_set_parent_factors}
				\EndForEach
			\EndForEach
		\EndForEach
		\State $D \gets$ Call \ac{lve} on $M'$ and $P(Q \mid R'_1 = r'_1, \ldots, R'_k = r'_k)$\; \label{line:pcfg_lci_call_lve}
		\State \Return $D$\; \label{line:pcfg_lci_return}
	\end{algorithmic}
\end{algorithm}

\subsection{Algorithm Description}
We next describe the \ac{lci} algorithm to compute the result of an interventional query $P(Q \mid do(R'_1 = r'_1, \allowbreak \ldots, \allowbreak R'_k = r'_k))$ in a \ac{pcfg} $M$ where every ground factor represents a conditional probability distribution.
\Cref{alg:pcfg_lci} displays the entire \ac{lci} algorithm, which we now discuss in detail.
First, in \cref{line:pcfg_lci_pf_splitting}, \ac{lci} splits the \acp{pf} in $M$ on the intervention variables $R'_i \in \{R'_1, \allowbreak \ldots, \allowbreak R'_k\}$ to obtain a modified \ac{pcfg} $M'$.
More specifically, \ac{lci} splits every \ac{pf} $\phi \in \boldsymbol \Phi$ for which there is an instance $\phi_j \in \gr(\phi)$ such that any intervention variable $R'_i \in \{R'_1, \allowbreak \ldots, \allowbreak R'_k\}$ is a child of $\phi_j$.
After the splitting procedure, the semantics of the model remains unchanged as the set of ground factors in $M'$ is still the same as the set of ground factors of the initial model $M$.
The only difference after splitting is that the ground factors are now arranged differently across the sets of ground instances.
Having completed the split of all respective \acp{pf}, \ac{lci} next changes the parent \acp{pf} of all intervention variables $R'_i \in \{R'_1, \allowbreak \ldots, \allowbreak R'_k\}$ to modify the underlying probability distribution encoded by $M'$ according to the semantics of the intervention $do(R'_1 = r'_1, \allowbreak \ldots, \allowbreak R'_k = r'_k)$ (\cref{line:pcfg_lci_outer_loop,line:pcfg_lci_middle_loop,line:pcfg_lci_inner_loop,line:pcfg_lci_set_parent_factors}).
\Ac{lci} changes the \acp{pf} in $M'$ according to \cref{prop:pcfg_lci_intervention} and thus, after the \acp{pf} have been changed, $M'$ encodes the interventional distribution under the intervention $do(R'_1 = r'_1, \allowbreak \ldots, \allowbreak R'_k = r'_k)$.
More specifically, as each intervention variable $R'_i$ is fixed to the value $r'_i$, all parent \acp{pf} $\phi'_j(R_{j_1}, \allowbreak \ldots, \allowbreak R_{j_z}) \in \Pa(M', R'_i)$ of $R'_i$ are altered such that all input sequences $(r_{j_1}, \ldots, r_{j_z})$ assigning $R'_i = r'_i$ map to the value one while all other input sequences map to zero.
Finally, \ac{lci} computes the result of the probabilistic query $P(Q \mid R'_1 = r'_1, \ldots, R'_k = r'_k)$ in the modified model $M'$, which is, according to \cref{prop:pcfg_lci_intervention}, equivalent to the result of the interventional query $P(Q \mid do(R'_1 = r'_1, \allowbreak \ldots, \allowbreak R'_k = r'_k))$ in the original model $M$.
To compute the result of $P(Q \mid R'_1 = r'_1, \ldots, R'_k = r'_k)$ in $M'$, \ac{lci} calls \ac{lve} on $P(Q \mid R'_1 = r'_1, \ldots, R'_k = r'_k)$ and $M'$ and then returns the result computed by \ac{lve} (\cref{line:pcfg_lci_call_lve,line:pcfg_lci_return}).
During this step, \ac{lve} (which originally operates on a \ac{pfg}) ignores the edge directions in $M'$.
Since the semantics of the underlying full joint probability distribution encoded by a \ac{pfg} are defined identically to the semantics of the probability distribution encoded by a \ac{pcfg}, \ac{lve} can also be applied to compute the result of probabilistic queries in a \ac{pcfg} (alternatively, a different lifted inference algorithm that works on a \ac{pfg} could be called as well).

\begin{example}[Lifted Causal Inference]
	Look at the \ac{pcfg} $M$ shown in \cref{fig:pcfg_example_pcfg} and consider the interventional query $P(Rev \mid do(Com(Bob) = \high))$.
	In accordance with the previous examples, we assume that $\domain{E} = \{ Alice, \allowbreak Bob, \allowbreak Charlie \}$.
	Since $Com(Bob)$ is a particular instance of $Com(E)$, we have to split the \ac{pf} $\phi_1(Com(E))_{| \top}$, which is a parent \ac{pf} of $Com(E)$.
	\Cref{fig:example_lci} shows the modified \ac{pcfg} $M'$ obtained after splitting $\phi_1(Com(E))_{| \top}$ on $Com(Bob)$.
	In $M'$, $\phi_1(Com(E))_{| \top}$ has been replaced by $\phi_1(Com(E))_{| C'}$ (the corresponding \ac{pf} node is $g'_1$) and $\phi_1(Com(E))_{| C''}$ (the corresponding \ac{pf} node is $g''_1$), where $C' = (E, \{ Bob \})$ and $C'' = (E, \{ Alice, \allowbreak Charlie \})$.
	To incorporate the semantics of the intervention $do(Com(Bob)) = \high$, \ac{lci} next modifies the parent \acp{pf} of $Com(Bob)$, i.e., \ac{lci} modifies $\phi_1(Com(E))_{| C'}$ in this example (since $\phi_1(Com(E))_{| C''}$ is constrained to $Alice$ and $Charlie$, $\phi_1(Com(E))_{| C''}$ is not a parent of $Com(Bob)$ and hence not changed).
	More specifically, $\phi_1(Com(E) = \high)_{| C'}$ is set to one and $\phi_1(Com(E) = \low)_{| C'}$ is set to zero.
	Finally, \ac{lci} runs \ac{lve} to compute $P(Rev \mid Com(Bob) = \high)$ in $M'$, which is equivalent to computing $P(Rev \mid do(Com(Bob) = \high))$ in the original model $M$.
\end{example}

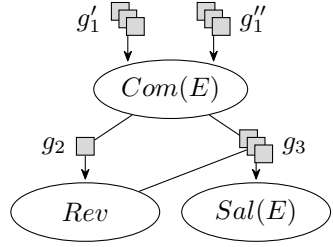
\begin{figure}[t]
	\centering
	\begin{tikzpicture}[rv/.append style={minimum height=2.2em, minimum width=5.6em}]
		\node[rv, inner sep = 1.8pt] (C) {$Com(E)$};

		\pfs{above left}{C}{1.5em and -0.8em}{180}{$g'_1$}{G1a}{G1}{G1b}
		\pfs{above right}{C}{1.5em and -0.8em}{0}{$g''_1$}{G1pa}{G1p}{G1pb}
		\factor{below left}{C}{1.0em and 0.8em}{180}{$g_2$}{G2}
		\pfs{below right}{C}{1.0em and 0.8em}{0}{$g_3$}{G3a}{G3}{G3b}

		\node[rv, below = 1.0em of G2] (R) {$Rev$};
		\node[rv, inner sep = 1.8pt, below = 1.0em of G3] (S) {$Sal(E)$};

		\node[coordinate] (G1toC) at ($(G1)-(0,1.7em)$) {};
		\node[coordinate] (G1ptoC) at ($(G1p)-(0,1.7em)$) {};

		\begin{pgfonlayer}{bg}
			\draw[arc] (G1) -- (G1toC);
			\draw[arc] (G1p) -- (G1ptoC);
			\draw (C) -- (G2);
			\draw (C) -- (G3);
			\draw[arc] (G2) -- (R);
			\draw (R) -- (G3);
			\draw[arc] (G3) -- (S);
		\end{pgfonlayer}
	\end{tikzpicture}
	\caption[A visualisation of the modified \ac{pcfg} obtained after altering the \ac{pcfg} shown in \cref{fig:pcfg_example_pcfg}.]{A visualisation of the modified \ac{pcfg} obtained after altering the \ac{pcfg} shown in \cref{fig:pcfg_example_pcfg} by splitting $\phi_1(Com(E))_{| \top}$ on $Com(Bob)$.}
	\label{fig:example_lci}
\end{figure}

\Ac{lci} is able to handle both interventions on a single (ground) \ac{rv} as well as interventions on a conjunction of multiple \acp{rv} efficiently.
In particular, when intervening on multiple indistinguishable \acp{rv} at the same time, \ac{lci} is able to treat those \acp{rv} as a group even after the intervention.
For instance, assume that in our running example, we would like to train multiple employees simultaneously, as a training program is mostly offered not only for a single employee but for a group of employees (here, training an employee $e \in \domain{E}$ corresponds to the intervention $do(Com(e) = \high)$).
Then, it is not necessary to split all trained employees into separate groups but instead it is sufficient to differentiate between trained employees and all remaining employees.
Formally, the intervention $do(R'_1 = r'_1, \allowbreak \ldots, \allowbreak R'_k = r'_k)$ on an arbitrary set of \acp{rv} $\{ R'_1, \allowbreak \ldots, \allowbreak R'_k \}$ can thus efficiently be handled by splitting the \acp{pf} in $M$ such that all $R'_i$ that are represented by the same \ac{prv} $A$ and that are set to the same value $r'_i \in \range{R'_i}$ remain grouped.
Specifically, \ac{lci} needs just a single split on the \acp{pf} per group and thus avoids manipulating the parents of each individual \ac{rv} separately.
In contrast, in a propositional model, every object has to be treated individually and therefore the parents for each \ac{rv} need to be manipulated separately.
Given the way we specified the semantics of an intervention in a \ac{pcfg}, it immediately follows that \ac{lci} correctly computes the effect of interventions.

\begin{proposition}
	The result computed by \ac{lci} (\cref{alg:pcfg_lci}) is the correct answer to the interventional query $P(Q \mid do(R'_1 = r'_1, \allowbreak \ldots, \allowbreak R'_k = r'_k))$ in the given \ac{pcfg} $M$.
\end{proposition}
\begin{proof}
	As \ac{lci} directly applies \cref{prop:pcfg_lci_intervention} by setting the parent factors of all intervention variables accordingly, the result of the interventional query $P(Q \mid do(R'_1 = r'_1, \allowbreak \ldots, \allowbreak R'_k = r'_k))$ in the original model $M$ is equivalent to the result of the probabilistic query $P(Q \mid R'_1 = r'_1, \allowbreak \ldots, \allowbreak R'_k = r'_k)$ in the modified model $M'$ computed by \ac{lci}.
\end{proof}

Moreover, by calling \ac{lve}, \ac{lci} allows for tractable probabilistic inference problems with respect to domain sizes of \acp{lv}.
Thus, \ac{lci} runs in polynomial time with respect to domain sizes of \acp{lv} for all \acp{pcfg} belonging to the class of domain-liftable models.
The class of domain-liftable models includes all \acp{pcfg} containing only \acp{pf} with at most two \acp{lv} and all \acp{pcfg} containing only \acp{prv} having at most one \ac{lv}~\citep{VanDenBroeck2011a}.

\begin{proposition}
	\Ac{lci} (\cref{alg:pcfg_lci}) allows for tractable probabilistic inference problems with respect to domain sizes of \acp{lv} for the class of domain-liftable models.
\end{proposition}
\begin{proof}
	If the input \ac{pcfg} $M$ for \ac{lci} belongs to the class of domain-liftable models, so does the modified model $M'$ obtained after splitting the \acp{pf} in $M$ and changing the parent factors of the intervention variables because newly introduced \acp{pf} contain identical \acp{lv} as the original \acp{pf} that were split.
	Consequently, the input for \ac{lve}, which is given by $M'$, belongs to the class of domain-liftable models and as \ac{lve} is complete for this model class~\citep{Taghipour2013b} (that is, \ac{lve} runs in polynomial time with respect to the domain sizes of the \acp{lv} in its input model for all combinations of queries, evidence, and models in this class), the call of \ac{lve} allows for tractable probabilistic inference problems with respect to domain sizes of \acp{lv} provided that $M$ belongs to the class of domain-liftable models.
	Furthermore, both the splitting procedure in \cref{line:pcfg_lci_pf_splitting} and the loops in \cref{line:pcfg_lci_outer_loop,line:pcfg_lci_middle_loop,line:pcfg_lci_inner_loop,line:pcfg_lci_set_parent_factors} of \cref{alg:pcfg_lci} do not influence the overall time complexity of \ac{lci} (as the loops iterate over potential tables that must be considered anyway during inference).
	Thus, \ac{lci} allows for tractable probabilistic inference problems with respect to domain sizes of \acp{lv} for the class of domain-liftable models.
\end{proof}

To summarise, \ac{lci} is a simple, yet effective algorithm to perform lifted causal inference.
\Ac{lci} can also handle queries with multiple query variables, provided that the lifted inference algorithm which is called in \cref{line:pcfg_lci_call_lve} can handle multiple query variables as well.
Next, we take a look at our experiments, which highlight the practical performance of \ac{lci} to compute the effect of interventions in a \ac{pcfg} on a lifted level.

\subsection{Experiments} \label{sec:pcfg_eval}
In this subsection, we evaluate the runtimes needed to compute the result of interventional queries in \acp{cbn}, \acp{cfg}, and \acp{pcfg}.
For our experiments, we use a slightly modified version of the \ac{pcfg} $M$ given in \cref{fig:pcfg_example_pcfg}, whose ground \ac{cfg} directly corresponds to a \ac{cbn}.
In addition to the \ac{pcfg} $M$, we also investigate runtimes for causal inference in the \ac{cfg} obtained by grounding $M$, and its equivalent \ac{cbn}.
To obtain the equivalent \ac{cbn}, we apply the transformation from directed \ac{fg} to \ac{cbn} proposed by \citet{Frey2003a}.
Hence, all three models, the \ac{pcfg}, the \ac{cfg}, and the \ac{cbn}, encode the same underlying full joint probability distribution.
As a remark, we note that the \ac{pcfg} used in our experiments to demonstrate the practical efficiency of lifted causal inference is rather small with four \acp{pf} and \acp{prv}, respectively, and the gain we obtain from lifted inference might further increase for models consisting of more \acp{prv}.
Our experiments can thus be seen as a proof of concept demonstrating the practical efficiency of lifted causal inference and more extensive experiments on \acp{pcfg} with various graph structures and on \acp{pcfg} with more \acp{prv} are left for future work.

\begin{figure}[t]
	\centering
	\begin{tikzpicture}[x=1pt,y=1pt]
		\definecolor{fillColor}{RGB}{255,255,255}
		\path[use as bounding box,fill=fillColor,fill opacity=0.00] (0,0) rectangle (260.17,115.63);
		\begin{scope}
		\path[clip] (  0.00,  0.00) rectangle (260.17,115.63);
		\definecolor{drawColor}{RGB}{255,255,255}
		\definecolor{fillColor}{RGB}{255,255,255}
		
		\path[draw=drawColor,line width= 0.6pt,line join=round,line cap=round,fill=fillColor] (  0.00,  0.00) rectangle (260.17,115.63);
		\end{scope}
		\begin{scope}
		\path[clip] ( 44.91, 30.69) rectangle (254.67,110.13);
		\definecolor{fillColor}{RGB}{255,255,255}
		
		\path[fill=fillColor] ( 44.91, 30.69) rectangle (254.67,110.13);
		\definecolor{drawColor}{RGB}{247,192,26}
		
		\path[draw=drawColor,line width= 0.6pt,line join=round] ( 57.25, 45.01) --
			( 77.81, 44.82) --
			( 98.38, 45.23) --
			(118.94, 45.89) --
			(139.51, 46.75) --
			(160.07, 48.72) --
			(180.64, 52.98) --
			(201.20, 55.97) --
			(221.77, 57.79) --
			(242.33, 62.08);
		\definecolor{drawColor}{RGB}{37,122,164}
		
		\path[draw=drawColor,line width= 0.6pt,dash pattern=on 2pt off 2pt ,line join=round] ( 57.25, 34.30) --
			( 77.81, 35.32) --
			( 98.38, 40.59) --
			(118.94, 47.19) --
			(139.51, 54.22) --
			(160.07, 62.81) --
			(180.64, 72.52) --
			(201.20, 83.15) --
			(221.77, 94.43) --
			(242.33,106.52);
		\definecolor{drawColor}{RGB}{78,155,133}
		
		\path[draw=drawColor,line width= 0.6pt,dash pattern=on 4pt off 2pt ,line join=round] ( 57.25, 43.60) --
			( 77.81, 45.68) --
			( 98.38, 48.57) --
			(118.94, 52.01) --
			(139.51, 56.57) --
			(160.07, 61.06) --
			(180.64, 66.57) --
			(201.20, 74.25) --
			(221.77, 84.47) --
			(242.33, 96.09);
		\definecolor{fillColor}{RGB}{37,122,164}
		
		\path[fill=fillColor] ( 57.25, 37.35) --
			( 59.89, 32.77) --
			( 54.61, 32.77) --
			cycle;
		
		\path[fill=fillColor] ( 77.81, 38.37) --
			( 80.46, 33.79) --
			( 75.17, 33.79) --
			cycle;
		
		\path[fill=fillColor] ( 98.38, 43.64) --
			(101.02, 39.06) --
			( 95.74, 39.06) --
			cycle;
		
		\path[fill=fillColor] (118.94, 50.24) --
			(121.59, 45.66) --
			(116.30, 45.66) --
			cycle;
		
		\path[fill=fillColor] (139.51, 57.28) --
			(142.15, 52.70) --
			(136.87, 52.70) --
			cycle;
		
		\path[fill=fillColor] (160.07, 65.86) --
			(162.72, 61.29) --
			(157.43, 61.29) --
			cycle;
		
		\path[fill=fillColor] (180.64, 75.57) --
			(183.28, 70.99) --
			(178.00, 70.99) --
			cycle;
		
		\path[fill=fillColor] (201.20, 86.20) --
			(203.85, 81.62) --
			(198.56, 81.62) --
			cycle;
		
		\path[fill=fillColor] (221.77, 97.48) --
			(224.41, 92.90) --
			(219.13, 92.90) --
			cycle;
		
		\path[fill=fillColor] (242.33,109.57) --
			(244.98,105.00) --
			(239.69,105.00) --
			cycle;
		\definecolor{fillColor}{RGB}{78,155,133}
		
		\path[fill=fillColor] ( 55.29, 41.64) --
			( 59.21, 41.64) --
			( 59.21, 45.57) --
			( 55.29, 45.57) --
			cycle;
		\definecolor{fillColor}{RGB}{247,192,26}
		
		\path[fill=fillColor] ( 57.25, 45.01) circle (  1.96);
		\definecolor{fillColor}{RGB}{78,155,133}
		
		\path[fill=fillColor] ( 75.85, 43.71) --
			( 79.78, 43.71) --
			( 79.78, 47.64) --
			( 75.85, 47.64) --
			cycle;
		\definecolor{fillColor}{RGB}{247,192,26}
		
		\path[fill=fillColor] ( 77.81, 44.82) circle (  1.96);
		\definecolor{fillColor}{RGB}{78,155,133}
		
		\path[fill=fillColor] ( 96.42, 46.60) --
			(100.34, 46.60) --
			(100.34, 50.53) --
			( 96.42, 50.53) --
			cycle;
		\definecolor{fillColor}{RGB}{247,192,26}
		
		\path[fill=fillColor] ( 98.38, 45.23) circle (  1.96);
		\definecolor{fillColor}{RGB}{78,155,133}
		
		\path[fill=fillColor] (116.98, 50.05) --
			(120.91, 50.05) --
			(120.91, 53.97) --
			(116.98, 53.97) --
			cycle;
		\definecolor{fillColor}{RGB}{247,192,26}
		
		\path[fill=fillColor] (118.94, 45.89) circle (  1.96);
		\definecolor{fillColor}{RGB}{78,155,133}
		
		\path[fill=fillColor] (137.55, 54.61) --
			(141.47, 54.61) --
			(141.47, 58.53) --
			(137.55, 58.53) --
			cycle;
		\definecolor{fillColor}{RGB}{247,192,26}
		
		\path[fill=fillColor] (139.51, 46.75) circle (  1.96);
		\definecolor{fillColor}{RGB}{78,155,133}
		
		\path[fill=fillColor] (158.11, 59.10) --
			(162.04, 59.10) --
			(162.04, 63.02) --
			(158.11, 63.02) --
			cycle;
		\definecolor{fillColor}{RGB}{247,192,26}
		
		\path[fill=fillColor] (160.07, 48.72) circle (  1.96);
		\definecolor{fillColor}{RGB}{78,155,133}
		
		\path[fill=fillColor] (178.68, 64.61) --
			(182.60, 64.61) --
			(182.60, 68.54) --
			(178.68, 68.54) --
			cycle;
		\definecolor{fillColor}{RGB}{247,192,26}
		
		\path[fill=fillColor] (180.64, 52.98) circle (  1.96);
		\definecolor{fillColor}{RGB}{78,155,133}
		
		\path[fill=fillColor] (199.24, 72.29) --
			(203.17, 72.29) --
			(203.17, 76.21) --
			(199.24, 76.21) --
			cycle;
		\definecolor{fillColor}{RGB}{247,192,26}
		
		\path[fill=fillColor] (201.20, 55.97) circle (  1.96);
		\definecolor{fillColor}{RGB}{78,155,133}
		
		\path[fill=fillColor] (219.81, 82.51) --
			(223.73, 82.51) --
			(223.73, 86.43) --
			(219.81, 86.43) --
			cycle;
		\definecolor{fillColor}{RGB}{247,192,26}
		
		\path[fill=fillColor] (221.77, 57.79) circle (  1.96);
		\definecolor{fillColor}{RGB}{78,155,133}
		
		\path[fill=fillColor] (240.37, 94.13) --
			(244.30, 94.13) --
			(244.30, 98.06) --
			(240.37, 98.06) --
			cycle;
		\definecolor{fillColor}{RGB}{247,192,26}
		
		\path[fill=fillColor] (242.33, 62.08) circle (  1.96);
		\end{scope}
		\begin{scope}
		\path[clip] (  0.00,  0.00) rectangle (260.17,115.63);
		\definecolor{drawColor}{RGB}{0,0,0}
		
		\path[draw=drawColor,line width= 0.6pt,line join=round] ( 44.91, 30.69) --
			( 44.91,110.13);
		
		\path[draw=drawColor,line width= 0.6pt,line join=round] ( 46.33,107.67) --
			( 44.91,110.13) --
			( 43.49,107.67);
		\end{scope}
		\begin{scope}
		\path[clip] (  0.00,  0.00) rectangle (260.17,115.63);
		\definecolor{drawColor}{gray}{0.30}
		
		\node[text=drawColor,anchor=base east,inner sep=0pt, outer sep=0pt, scale=  0.88] at ( 39.96, 42.93) {10};
		
		\node[text=drawColor,anchor=base east,inner sep=0pt, outer sep=0pt, scale=  0.88] at ( 39.96, 63.01) {100};
		
		\node[text=drawColor,anchor=base east,inner sep=0pt, outer sep=0pt, scale=  0.88] at ( 39.96, 83.08) {1000};
		
		\node[text=drawColor,anchor=base east,inner sep=0pt, outer sep=0pt, scale=  0.88] at ( 39.96,103.16) {10000};
		\end{scope}
		\begin{scope}
		\path[clip] (  0.00,  0.00) rectangle (260.17,115.63);
		\definecolor{drawColor}{gray}{0.20}
		
		\path[draw=drawColor,line width= 0.6pt,line join=round] ( 42.16, 45.96) --
			( 44.91, 45.96);
		
		\path[draw=drawColor,line width= 0.6pt,line join=round] ( 42.16, 66.04) --
			( 44.91, 66.04);
		
		\path[draw=drawColor,line width= 0.6pt,line join=round] ( 42.16, 86.11) --
			( 44.91, 86.11);
		
		\path[draw=drawColor,line width= 0.6pt,line join=round] ( 42.16,106.19) --
			( 44.91,106.19);
		\end{scope}
		\begin{scope}
		\path[clip] (  0.00,  0.00) rectangle (260.17,115.63);
		\definecolor{drawColor}{RGB}{0,0,0}
		
		\path[draw=drawColor,line width= 0.6pt,line join=round] ( 44.91, 30.69) --
			(254.67, 30.69);
		
		\path[draw=drawColor,line width= 0.6pt,line join=round] (252.21, 29.26) --
			(254.67, 30.69) --
			(252.21, 32.11);
		\end{scope}
		\begin{scope}
		\path[clip] (  0.00,  0.00) rectangle (260.17,115.63);
		\definecolor{drawColor}{gray}{0.20}
		
		\path[draw=drawColor,line width= 0.6pt,line join=round] ( 57.25, 27.94) --
			( 57.25, 30.69);
		
		\path[draw=drawColor,line width= 0.6pt,line join=round] ( 77.81, 27.94) --
			( 77.81, 30.69);
		
		\path[draw=drawColor,line width= 0.6pt,line join=round] ( 98.38, 27.94) --
			( 98.38, 30.69);
		
		\path[draw=drawColor,line width= 0.6pt,line join=round] (118.94, 27.94) --
			(118.94, 30.69);
		
		\path[draw=drawColor,line width= 0.6pt,line join=round] (139.51, 27.94) --
			(139.51, 30.69);
		
		\path[draw=drawColor,line width= 0.6pt,line join=round] (160.07, 27.94) --
			(160.07, 30.69);
		
		\path[draw=drawColor,line width= 0.6pt,line join=round] (180.64, 27.94) --
			(180.64, 30.69);
		
		\path[draw=drawColor,line width= 0.6pt,line join=round] (201.20, 27.94) --
			(201.20, 30.69);
		
		\path[draw=drawColor,line width= 0.6pt,line join=round] (221.77, 27.94) --
			(221.77, 30.69);
		
		\path[draw=drawColor,line width= 0.6pt,line join=round] (242.33, 27.94) --
			(242.33, 30.69);
		\end{scope}
		\begin{scope}
		\path[clip] (  0.00,  0.00) rectangle (260.17,115.63);
		\definecolor{drawColor}{gray}{0.30}
		
		\node[text=drawColor,anchor=base,inner sep=0pt, outer sep=0pt, scale=  0.88] at ( 57.25, 19.68) {8};
		
		\node[text=drawColor,anchor=base,inner sep=0pt, outer sep=0pt, scale=  0.88] at ( 77.81, 19.68) {16};
		
		\node[text=drawColor,anchor=base,inner sep=0pt, outer sep=0pt, scale=  0.88] at ( 98.38, 19.68) {32};
		
		\node[text=drawColor,anchor=base,inner sep=0pt, outer sep=0pt, scale=  0.88] at (118.94, 19.68) {64};
		
		\node[text=drawColor,anchor=base,inner sep=0pt, outer sep=0pt, scale=  0.88] at (139.51, 19.68) {128};
		
		\node[text=drawColor,anchor=base,inner sep=0pt, outer sep=0pt, scale=  0.88] at (160.07, 19.68) {256};
		
		\node[text=drawColor,anchor=base,inner sep=0pt, outer sep=0pt, scale=  0.88] at (180.64, 19.68) {512};
		
		\node[text=drawColor,anchor=base,inner sep=0pt, outer sep=0pt, scale=  0.88] at (201.20, 19.68) {1024};
		
		\node[text=drawColor,anchor=base,inner sep=0pt, outer sep=0pt, scale=  0.88] at (221.77, 19.68) {2048};
		
		\node[text=drawColor,anchor=base,inner sep=0pt, outer sep=0pt, scale=  0.88] at (242.33, 19.68) {4096};
		\end{scope}
		\begin{scope}
		\path[clip] (  0.00,  0.00) rectangle (260.17,115.63);
		\definecolor{drawColor}{RGB}{0,0,0}
		
		\node[text=drawColor,anchor=base,inner sep=0pt, outer sep=0pt, scale=  1.10] at (149.79,  7.64) {Number of employees $d$};
		\end{scope}
		\begin{scope}
		\path[clip] (  0.00,  0.00) rectangle (260.17,115.63);
		\definecolor{drawColor}{RGB}{0,0,0}
		
		\node[text=drawColor,rotate= 90.00,anchor=base,inner sep=0pt, outer sep=0pt, scale=  1.10] at ( 13.08, 70.41) {time (ms)};
		\end{scope}
		\begin{scope}
		\path[clip] (  0.00,  0.00) rectangle (260.17,115.63);
		
		\path[] ( 41.63, 63.09) rectangle (119.51,117.45);
		\end{scope}
		\begin{scope}
		\path[clip] (  0.00,  0.00) rectangle (260.17,115.63);
		\definecolor{drawColor}{RGB}{247,192,26}
		
		\path[draw=drawColor,line width= 0.6pt,line join=round] ( 48.57,104.72) -- ( 60.13,104.72);
		\end{scope}
		\begin{scope}
		\path[clip] (  0.00,  0.00) rectangle (260.17,115.63);
		\definecolor{fillColor}{RGB}{247,192,26}
		
		\path[fill=fillColor] ( 54.35,104.72) circle (  1.96);
		\end{scope}
		\begin{scope}
		\path[clip] (  0.00,  0.00) rectangle (260.17,115.63);
		\definecolor{drawColor}{RGB}{37,122,164}
		
		\path[draw=drawColor,line width= 0.6pt,dash pattern=on 2pt off 2pt ,line join=round] ( 48.57, 90.27) -- ( 60.13, 90.27);
		\end{scope}
		\begin{scope}
		\path[clip] (  0.00,  0.00) rectangle (260.17,115.63);
		\definecolor{fillColor}{RGB}{37,122,164}
		
		\path[fill=fillColor] ( 54.35, 93.32) --
			( 56.99, 88.74) --
			( 51.71, 88.74) --
			cycle;
		\end{scope}
		\begin{scope}
		\path[clip] (  0.00,  0.00) rectangle (260.17,115.63);
		\definecolor{drawColor}{RGB}{78,155,133}
		
		\path[draw=drawColor,line width= 0.6pt,dash pattern=on 4pt off 2pt ,line join=round] ( 48.57, 75.82) -- ( 60.13, 75.82);
		\end{scope}
		\begin{scope}
		\path[clip] (  0.00,  0.00) rectangle (260.17,115.63);
		\definecolor{fillColor}{RGB}{78,155,133}
		
		\path[fill=fillColor] ( 52.39, 73.85) --
			( 56.31, 73.85) --
			( 56.31, 77.78) --
			( 52.39, 77.78) --
			cycle;
		\end{scope}
		\begin{scope}
		\path[clip] (  0.00,  0.00) rectangle (260.17,115.63);
		\definecolor{drawColor}{RGB}{0,0,0}
		
		\node[text=drawColor,anchor=base west,inner sep=0pt, outer sep=0pt, scale=  0.80] at ( 67.08,101.97) {LVE (PCFG)};
		\end{scope}
		\begin{scope}
		\path[clip] (  0.00,  0.00) rectangle (260.17,115.63);
		\definecolor{drawColor}{RGB}{0,0,0}
		
		\node[text=drawColor,anchor=base west,inner sep=0pt, outer sep=0pt, scale=  0.80] at ( 67.08, 87.52) {VE (CBN)};
		\end{scope}
		\begin{scope}
		\path[clip] (  0.00,  0.00) rectangle (260.17,115.63);
		\definecolor{drawColor}{RGB}{0,0,0}
		
		\node[text=drawColor,anchor=base west,inner sep=0pt, outer sep=0pt, scale=  0.80] at ( 67.08, 73.06) {VE (CFG)};
		\end{scope}
	\end{tikzpicture}		
	\caption{A comparison of the runtimes required to compute interventional distributions on different graphical models encoding equivalent full joint probability distributions.}
	\label{fig:pcfg_plot_eval}
\end{figure}
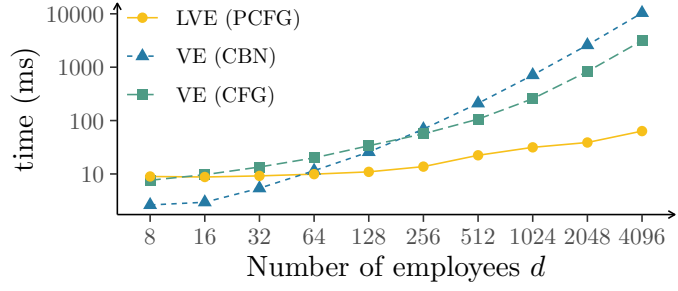

We test the required runtime to compute the result of an interventional query for each of the three graphical models on different graph sizes by setting the domain size of the employees to $d \in \{ 8, \allowbreak 16, \allowbreak 32, \allowbreak 64, \allowbreak 128, \allowbreak 256, \allowbreak 512, \allowbreak 1024, \allowbreak 2048, \allowbreak 4096 \}$, that is, $\abs{\domain{E}} = d$.
\Cref{fig:pcfg_plot_eval} shows the runtimes needed to compute the result of the probabilistic query in the modified graph when running \ac{ve} on the \ac{cfg}, \ac{ve} on the \ac{cbn}, and \ac{lve} on the \ac{pcfg}.
As we have seen, \ac{ve} algorithm is the propositional counterpart of \ac{lve} and operates on a propositional (ground) model, such as a \ac{cbn} or an \ac{cfg}.
Consequently, \ac{ve} considers every object (e.g., every employee) individually for computations, independent of whether objects are indistinguishable or not.
In contrast, \ac{lve} treats indistinguishable objects as a group by using a representative for computations instead of considering each of those objects separately.
The results emphasise that the \ac{lci} algorithm, which internally exploits \ac{lve}, overcomes scalability issues for large domain sizes as the runtime of \ac{lve}, in contrast to the runtimes of \ac{ve} on the \ac{cbn} and the \ac{cfg}, does not exponentially increase for increasing values of $d$ (y-axis is log-scaled).
Even though the splitting of \acp{pf} results in a less compressed lifted representation, it becomes evident that the performance of \ac{lve} is not significantly affected by the splitting.

While \ac{lci} in combination with a given \ac{pcfg} solves the problem of efficiently computing the effect of interventions on a lifted level, in practice, we often face the problem of not knowing all the underlying causal relationships.
Thus, in the upcoming section, we relax the assumption of knowing all causal relationships by allowing for partial causal knowledge.
In particular, we introduce \acp{pdpcfg} a generalisation of \acp{pcfg} and investigate the implications of not knowing all causal relationships for answering interventional queries.
Moreover, we have hitherto assumed that every factor in a \ac{pcfg} encodes a conditional probability distribution and we also abandon this assumption in the next section.

\section{Partially Directed Parametric Causal Factor Graphs} \label{sec:pdpcfg_pdpcfg}
We now move on to define a \ac{pdpcfg} as a lifted representation that is able to incorporate partial causal knowledge by combining a \ac{pfg} with a partially directed graph to model causal relationships.
The major advantage of a \ac{pdpcfg} over a \ac{pcfg} is that not all causal relationships between the involved \acp{rv} need to be known, thereby reducing the amount of prior knowledge required and thus making the formalism more suitable for many practical settings.

\begin{definition}[Partially Directed Parametric Causal Factor Graph]
	A \emph{\ac{pdpcfg}} $M = (\boldsymbol V, \boldsymbol E, \boldsymbol \Phi)$ consists of a partially directed graph $(\boldsymbol V, \boldsymbol E)$ with node set $\boldsymbol V = \boldsymbol A \cup \boldsymbol G$ and edge set $\boldsymbol E \subseteq \boldsymbol A \times \boldsymbol G$.
	The set of nodes $\boldsymbol V = \boldsymbol A \cup \boldsymbol G$ is partitioned into a set of \acp{prv} $\boldsymbol A = \{ A_1, \ldots, A_n \}$ and a set of \ac{pf} names (\ac{pf} nodes) $\boldsymbol G = \{ g_1, \ldots, g_m \}$.
	For every \ac{pf} name $g_j \in \boldsymbol G$, there is a function definition (\ac{pf}) $\phi_j(\mathcal A_j)_{| C} \in \boldsymbol \Phi$ with $\mathcal A_j$ being a sequence of \acp{prv} from $\boldsymbol A$ and $C$ being a constraint on the \acp{lv} of $\mathcal A_j$ such that $\phi\colon \times_{A \in \mathcal A_j} \range{A} \mapsto \mathbb{R}_{\geq 0}$ maps range values in $\mathcal A_j$ to a positive real number (potential).
	As usual, in every function definition, at least one potential has to be non-zero and we may omit $| \top$ in $\phi_j(\mathcal A_j)_{| \top}$.
	For each \ac{pf} name $g_j \in \boldsymbol G$ with corresponding function definition $\phi_j(\mathcal A_j)_{| C} \in \boldsymbol \Phi$, there is either an undirected edge $\{ A_i, g_j \} \in \boldsymbol E$ or a directed edge $(g_j, A_i) \in \boldsymbol E$ for every \ac{prv} $A_i \in \mathcal A_j$.
	We stipulate that for every \ac{pf} node $g_j \in \boldsymbol G$, there is at most one outgoing directed edge $(g_j, A_i) \in \boldsymbol E$ among the edges incident to $g_j$.
	In case a \ac{pf} $\phi_j \in \boldsymbol \Phi$ has only a single argument $A \in \boldsymbol A$, its corresponding \ac{pf} node $g_j \in \boldsymbol G$ is connected to $A$ via a directed edge $(g_j, A) \in \boldsymbol E$.
	Furthermore, each directed edge $A_k - g_j \to A_i$ from a \ac{prv} $A_k \in \boldsymbol A$ to a \ac{prv} $A_i \in \boldsymbol A$ via a \ac{pf} node $g_j \in \boldsymbol G$ corresponds to a direct causal relationship between $A_k$ and $A_i$.
	The directed edges in $\boldsymbol E$ are not allowed to form any directed cycles, i.e., $\boldsymbol E$ contains no sequence of edges $\{ A_1, \allowbreak g_1 \}, \allowbreak (g_1, \allowbreak A_2), \allowbreak \ldots, \allowbreak \{ A_{k - 1}, \allowbreak g_k \}, \allowbreak (g_k, \allowbreak A_1)$ starting from an arbitrary \ac{prv} $A_1 \in \boldsymbol A$ such that the sequence ends again at $A_1$ while every second edge in the sequence is directed and the sequence follows the arrow directions of the directed edges.
	The semantics of $M$ is given by grounding with respect to constraints and building a full joint distribution over $\boldsymbol R = \gr(\boldsymbol A)$.
	The joint potential for an assignment $\boldsymbol R = \boldsymbol r$ is defined as
	\begin{align}
		\psi_M(\boldsymbol R = \boldsymbol r) = \prod_{\phi_j \in \boldsymbol \Phi} \prod_{\phi_k \in \gr(\phi_j)} \phi_k(\mathcal R_k = \boldsymbol r_k),
	\end{align}
	where $\boldsymbol r_k$ is a projection of $\boldsymbol r$ to the argument list $\mathcal R_k$ of $\phi_k$.
	The normalised joint potential then yields the full joint probability distribution over $\boldsymbol R$ that is encoded by $M$, that is,
	\begin{align}
		P_M(\boldsymbol R = \boldsymbol r) = \frac{1}{Z} \psi_M(\boldsymbol R = \boldsymbol r),
	\end{align}
	where $Z$ is the normalisation constant, defined as
	\begin{align}
		Z = \sum_{\boldsymbol r \in \times_{R \in \boldsymbol R} \range{R}} \psi_M(\boldsymbol R = \boldsymbol r).
	\end{align}
\end{definition}

A \ac{pdpcfg} thus offers the possibility to omit edge directions if no information about the underlying causal relationships is available.
Grounding a \ac{pdpcfg} $M$ yields a partially directed \ac{cfg}, which encodes the same underlying full joint probability distribution as $M$.
If all causal relationships are known (and hence all \ac{pf} nodes have an outgoing edge), a \ac{pdpcfg} is identical to a \ac{pcfg}.
In a \ac{pdpcfg} $M = (\boldsymbol A \cup \boldsymbol G, \boldsymbol E, \boldsymbol \Phi)$, we follow the same notations for the parent \acp{prv} $\Pa_{\boldsymbol A}(M, A)$ of a \ac{prv} $A \in \boldsymbol A$, the parent \acp{prv} $\Pa(M, g)$ of a \ac{pf} node $g \in \boldsymbol G$, the child \acp{prv} $\Ch_{\boldsymbol A}(M, A)$ of a \ac{prv} $A \in \boldsymbol A$, the child \acp{prv} $\Ch(M, g)$ of a \ac{pf} node $g \in \boldsymbol G$, the descendant \acp{prv} $\De_{\boldsymbol A}(M, A)$ of a \ac{prv} $A \in \boldsymbol A$, and the descendant \acp{prv} $\De(M, g)$ of a \ac{pf} node $g \in \boldsymbol G$ as in a \ac{pcfg}.
Additionally, we define the (undirected) neighbour \acp{prv} of a \ac{prv} $A \in \boldsymbol A$ in $M$ as $\Ne_{\boldsymbol A}(M, A) = \{ A' \in \boldsymbol A \mid \exists g \in \boldsymbol G \colon \Ch(M, g) = \emptyset \land \{ g, A' \} \in \boldsymbol E \land \{ g, A \} \in \boldsymbol E \}$.
In contrast to a \ac{pcfg}, the set of child \acp{prv} $\Ch(M, g)$ of a \ac{pf} node $g \in \boldsymbol G$ in a \ac{pdpcfg} $M$ may be empty.
More specifically, as every \ac{pf} node has at most one outgoing edge, it holds that $\abs{\Ch(M, g)} \leq 1$ for every \ac{pf} node $g \in \boldsymbol G$.
If a \ac{pf} node $g \in \boldsymbol G$ corresponds to a \ac{pf} with a single argument, it always holds that $\abs{\Ch(M, g)} = 1$.
From now on, we also drop the assumption that every ground factor $\phi_j(R_{j_1}, \allowbreak \ldots, \allowbreak R_{j_z}) \in \gr(\boldsymbol \Phi)$ in a \ac{pdpcfg} $M$ encodes a conditional probability distribution $P(R_{j_z} \mid R_{j_1}, \allowbreak \ldots, \allowbreak R_{j_{z - 1}})$.
Let us next consider a modified version of our running example, where the underlying causal relationships are only partially known, resulting in a \ac{pdpcfg} instead of a fully directed \ac{pcfg}.

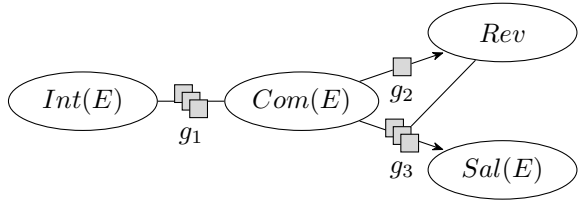
\begin{figure}[t]
	\centering
	\begin{tikzpicture}[rv/.append style={minimum height=2.2em, minimum width=5.6em}]
		\node[rv, draw, inner sep = 1.8pt] (I) {$Int(E)$};
		\node[rv, inner sep = 1.8pt, right = 2.5em of I] (C) {$Com(E)$};
		\node[rv, above right = 1.0em and 3.5em of C] (R) {$Rev$};
		\node[rv, inner sep = 1.8pt, below right = 1.0em and 3.5em of C] (S) {$Sal(E)$};

		\pfsat{$(I.east)!0.5!(C.west)$}{270}{$g_1$}{G1a}{G1}{G1b}
		\factorat{$(C)!0.5!(R)$}{270}{$g_2$}{G2}
		\pfsat{$(C)!0.5!(S)$}{270}{$g_3$}{G3a}{G3}{G3b}

		\begin{pgfonlayer}{bg}
			\draw (I) -- (G1);
			\draw (G1) -- (C);
			\draw (C) -- (G2);
			\draw[arc] (G2) -- (R);
			\draw (C) -- (G3);
			\draw (R) -- (G3);
			\draw[arc] (G3) -- (S);
		\end{pgfonlayer}
	\end{tikzpicture}
	\caption{A \ac{pdpcfg} that extends the \ac{pcfg} depicted in \cref{fig:pcfg_example_pcfg} in the sense that an additional \ac{prv} $Int(E)$ has been added to the model. We omit the specification of the potential tables of the (par)factors for brevity.}
	\label{fig:pdpcfg_example_pdpcfg}
\end{figure}

\begin{example}[Partially Directed Parametric Causal Factor Graph]
	\Cref{fig:pdpcfg_example_pdpcfg} depicts a \ac{pdpcfg} $M$, which extends the \ac{pcfg} given in \cref{fig:pcfg_example_pcfg}.
	In particular, there is an additional \ac{prv} $Int(E)$ in $M$, which represents the intelligence of an employee.
	Moreover, for the sake of the example, there is no information available about the causal relationship between $Int(E)$ and $Com(E)$.
	As $g_1$ has no outgoing directed edge, we have $\Ne_{\boldsymbol A}(M, Com(E)) = \{ Int(E) \}$ and $\Ne_{\boldsymbol A}(M, Int(E)) = \{ Com(E) \}$, whereas the remaining \acp{prv} have no undirected neighbour \acp{prv}.
	We omit the set notation of $M$ for brevity.
\end{example}

Separation in a \ac{pdpcfg} is defined as in a \ac{pcfg}, that is, the conditions specifying when a path is blocked are identical.
Again, it is also possible to check whether \acp{prv} (instead of ground \acp{rv}) are conditionally independent in a highly efficient manner on a lifted level in a \ac{pdpcfg}.
In a \ac{pdpcfg}, every \ac{prv} $A(\mathcal L)_{| C}$ is represented by a single variable node and thus, checking for conditional independence statements that involve $A$ can be done by looking at this single variable node instead of taking into account all groundings of $A$ individually.

In accordance with our previous assumptions, whenever we deal with a \ac{pdpcfg} in this article, we demand that whenever a probability distribution $P$ is modelled using a \ac{pdpcfg} $M$, $P$ satisfies the global Markov property with respect to $M$.
We further stipulate that all directed edges in a \ac{pdpcfg} $M$ are causal and hence accurately represent causal relationships between the involved \acp{rv}.

We next show how the computation of the effect of interventions can efficiently be realised in a \ac{pdpcfg}.
An important challenge is that the effect of an intervention might differ depending on the actual causal relationships between the \acp{rv}.
As there might be multiple possible causal explanations and we do not know the correct one, it is not always possible to uniquely determine the effect of an intervention.

The semantics of an intervention is defined on a fully directed graph.
In particular, the interventional distribution is defined as a factorisation over the conditional probability distributions of all \acp{rv}, which are no intervention variables, given their parents.
A \ac{pdpcfg}, however, might contain \ac{pf} nodes without any outgoing directed edges, thereby possibly leading to unknown sets of parents for some \acp{rv}.
Thus, when computing the effect of an intervention, we have to take all possible parent sets of the intervention variables into account.
In general, not all combinations of orienting the undirected edges in a \ac{pdpcfg} are consistent with the conditional independence statements holding in the underlying probability distribution.
More specifically, every \ac{pdpcfg} $M$ represents a set of fully directed \acp{pcfg} obtained by orienting the undirected edges in $M$ such that every \ac{pf} node has exactly one outgoing directed edge and the resulting model entails the same conditional independence statements as $M$.
We formalise this concept in the following definition.

\begin{definition}[Consistent Extension]
	Let $M = (\boldsymbol A \cup \boldsymbol G, \allowbreak \boldsymbol E, \allowbreak \boldsymbol \Phi)$ denote a \ac{pdpcfg}.
	A \ac{pcfg} $M' = (\boldsymbol A \cup \boldsymbol G, \allowbreak \boldsymbol E', \allowbreak \boldsymbol \Phi)$ is a \emph{consistent extension} of $M$ if
	\begin{enumerate}
		\item every directed edge $(g, A) \in \boldsymbol E$ is also in $\boldsymbol E'$,
		\item for every \ac{pf} node $g \in \boldsymbol G$ with $\Ch(M, g) = \emptyset$, exactly one edge $\{ A, g \} \in \boldsymbol E$ is replaced by an edge $(g, A)$ in $\boldsymbol E'$, and
		\item $M'$ entails the same conditional independence statements as $M$.
	\end{enumerate}
	We denote the set of all consistent extensions of $M$ as $[M]$.
\end{definition}

In other words, a consistent extension of a \ac{pdpcfg} $M$ is a \ac{pcfg} obtained by orienting the undirected edges in $M$ such that every \ac{pf} node has exactly one outgoing directed edge and the implied conditional independence statements in the extension remain the same as in $M$.
The set of consistent extensions $[M]$ of a \ac{pdpcfg} $M$ might be empty.
If $[M] \neq \emptyset$, we say that $M$ is extendable.
The concept of a consistent extension is closely related to the concept of a Markov equivalence class, which is a set of directed acyclic graphs that entail the same conditional independence statements~\citep{Verma1990a,Andersson1997a} (and thus, the set of consistent extensions of a partially directed acyclic graph is a subset of a Markov equivalence class).

\begin{example}[Consistent Extension]
	Consider again the \ac{pdpcfg} $M$ depicted in \cref{fig:pdpcfg_example_pdpcfg}.
	The set of consistent extensions $[M]$ of $M$ contains two fully directed \acp{pdpcfg} $M_1$ and $M_2$, which are illustrated in \cref{fig:pdpcfg_example_consistent_extensions_01} and \cref{fig:pdpcfg_example_consistent_extensions_02}, respectively.
	In $M_1$, the edge $g_1 - Com(E)$ has been replaced by an edge $g_1 \to Com(E)$ and in $M_2$, the edge $Int(E) - g_1$ has been replaced by an edge $g_1 \to Int(E)$.
	Both $M_1$ and $M_2$ entail the same conditional independence statements as $M$ and could possibly model the correct underlying causal relationships but we do not know whether $M_1$ or $M_2$ is actually the correct model.
\end{example}

\begin{figure}[t]
	\centering
	\begin{subfigure}{0.49\textwidth}
		\centering
		\resizebox{\textwidth}{!}{
			\begin{tikzpicture}[rv/.append style={minimum height=2.2em, minimum width=5.6em}]
				\node[rv, draw, inner sep = 1.8pt] (I) {$Int(E)$};
				\node[rv, inner sep = 1.8pt, right = 2.5em of I] (C) {$Com(E)$};
				\node[rv, above right = 1.0em and 3.5em of C] (R) {$Rev$};
				\node[rv, inner sep = 1.8pt, below right = 1.0em and 3.5em of C] (S) {$Sal(E)$};
			
				\pfsat{$(I.east)!0.45!(C.west)$}{270}{$g_1$}{G1a}{G1}{G1b}
				\factorat{$(C)!0.5!(R)$}{270}{$g_2$}{G2}
				\pfsat{$(C)!0.5!(S)$}{270}{$g_3$}{G3a}{G3}{G3b}
			
				\begin{pgfonlayer}{bg}
					\draw (I) -- (G1);
					\draw[arc] (G1) -- (C);
					\draw (C) -- (G2);
					\draw[arc] (G2) -- (R);
					\draw (C) -- (G3);
					\draw (R) -- (G3);
					\draw[arc] (G3) -- (S);
				\end{pgfonlayer}
			\end{tikzpicture}
		}
		\caption{}
		\label{fig:pdpcfg_example_consistent_extensions_01}
	\end{subfigure}
	\begin{subfigure}{0.49\textwidth}
		\centering
		\resizebox{\textwidth}{!}{
			\begin{tikzpicture}[rv/.append style={minimum height=2.2em, minimum width=5.6em}]
				\node[rv, draw, inner sep = 1.8pt] (I) {$Int(E)$};
				\node[rv, inner sep = 1.8pt, right = 2.5em of I] (C) {$Com(E)$};
				\node[rv, above right = 1.0em and 3.5em of C] (R) {$Rev$};
				\node[rv, inner sep = 1.8pt, below right = 1.0em and 3.5em of C] (S) {$Sal(E)$};
			
				\pfsat{$(I.east)!0.55!(C.west)$}{270}{$g_1$}{G1a}{G1}{G1b}
				\factorat{$(C)!0.5!(R)$}{270}{$g_2$}{G2}
				\pfsat{$(C)!0.5!(S)$}{270}{$g_3$}{G3a}{G3}{G3b}
			
				\begin{pgfonlayer}{bg}
					\draw[arc] (G1) -- (I);
					\draw (G1) -- (C);
					\draw (C) -- (G2);
					\draw[arc] (G2) -- (R);
					\draw (C) -- (G3);
					\draw (R) -- (G3);
					\draw[arc] (G3) -- (S);
				\end{pgfonlayer}
			\end{tikzpicture}
		}
		\caption{}
		\label{fig:pdpcfg_example_consistent_extensions_02}
	\end{subfigure}
	\caption{A graphical illustration of the set of consistent extensions $[M] = \{ M_1, M_2 \}$ of the \ac{pdpcfg} $M$ shown in \cref{fig:pdpcfg_example_pdpcfg}. (a) shows the \ac{pcfg} $M_1$, where $Com(E) - g_1$ has been oriented as $g_1 \to Com(E)$, and (b) shows the \ac{pcfg} $M_2$, where $Int(E) - g_1$ has been oriented as $g_1 \to Int(E)$.}
	\label{fig:pdpcfg_example_consistent_extensions}
\end{figure}
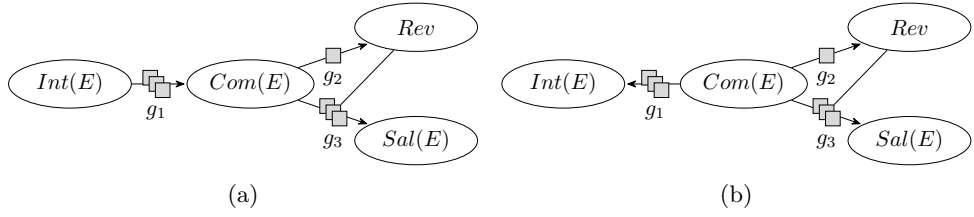

We remark that the definition of a consistent extension refers to edges in the \ac{pdpcfg} $M$ instead of referring to edges in the ground model $\gr(M)$.
Thus, we assume that every edge $\{ A, g \}$ in $M$ represents a set of edges in $\gr(M)$ such that all of the edges in this set are oriented in the same way.
For instance, in the ground graph of our running example, we do not allow any orientation where, e.g., $IntA - f_1 \to ComA$ and $IntB \gets f_2 - ComB$ occur at the same time.
As we assume that $Alice$ and $Bob$ are indistinguishable, we also assume identical edge orientations for their corresponding \acp{rv}.
Defining consistent extensions on the lifted level is, however, not necessary to apply the approaches presented here.
If wanted, the set of consistent extensions of a \ac{pdpcfg} $M$ can also be defined with respect to the ground model of $M$ (however, a definition with respect to the ground model is rather unintuitive as objects are not really indistinguishable if surrounding edges are oriented differently).

To determine the effect of an intervention, we need to know the parents of the intervention variables.
As any \ac{pdpcfg} represents a set of consistent extensions, there are various possible parent sets for the intervention variables in general.
Fortunately, we do not always have to consider all consistent extensions of a given \ac{pdpcfg} $M$ to compute the effect of an intervention because there might be consistent extensions with identical parent sets for the intervention variables, thereby leading to the same effect of the intervention.
Therefore, in case all parents of the \acp{rv} on which we intervene are known, we can uniquely determine the effect of an intervention even if there are still undirected edges present in $M$.
This result has been shown for propositional partially directed acyclic graphs~\citep{Maathuis2009a,Nandy2017a} and we now transfer this result to \acp{pdpcfg}.

\begin{theorem} \label{th:pdpcfg_known_parents}
	Let $M = (\boldsymbol A \cup \boldsymbol G, \allowbreak \boldsymbol E, \allowbreak \boldsymbol \Phi)$ denote a \ac{pdpcfg}, let $\boldsymbol R = \gr(\boldsymbol A) = \{ R_1, \allowbreak \ldots, \allowbreak R_n \}$ and let $do(R'_1 = r'_1, \allowbreak \ldots, \allowbreak R'_k = r'_k)$ be an intervention on $\{ R'_1, \allowbreak \ldots, \allowbreak R'_k \} \subseteq \boldsymbol R$.
	If $\Ne_{\boldsymbol R}(\gr(M), R'_1) = \emptyset, \allowbreak \ldots, \allowbreak \Ne_{\boldsymbol R}(\gr(M), R'_k) = \emptyset$, then the interventional distribution $P_{M'}(R_1 = r_1, \allowbreak \ldots, \allowbreak R_n = r_n \mid do(R'_1 = r'_1, \allowbreak \ldots, \allowbreak R'_k = r'_k))$ under the intervention $do(R'_1 = r'_1, \allowbreak \ldots, \allowbreak R'_k = r'_k)$ is identical in all consistent extensions $M' \in [M]$ of $M$.
\end{theorem}
\begin{proof}
	The interventional distribution of a \ac{pcfg} $M' \in [M]$ is given by (\cref{def:pcfg_interventional_distribution_pcfg}):
	\begin{align*}
		P_{M'}&(R_1 = r_1, \ldots, R_n = r_n \mid do(R'_1 = r'_1, \ldots, R'_k = r'_k)) \\
		&=
		\begin{cases}
			\prod\limits_{R_i \in \{ R_1, \ldots, R_n \} \setminus \{ R'_1, \ldots, R'_k \}} P(r_i \mid \pa_{\boldsymbol R}(\gr(M'), R_i)) & \text{if } \forall j \in \{1, \ldots, k\}\colon r_j = r'_j \\
			0 & \text{otherwise}.
		\end{cases}
	\end{align*}
	Given that $\Ne_{\boldsymbol R}(\gr(M), R'_1) = \emptyset, \ldots, \Ne_{\boldsymbol R}(\gr(M), R'_k) = \emptyset$, the parents $\Pa_{\boldsymbol R}(\gr(M), R'_1), \allowbreak \ldots, \allowbreak \Pa_{\boldsymbol R}(\gr(M), R'_k)$ of $R'_1, \ldots, R'_k$ in $M$ are fully known and identical in all $M' \in [M]$, i.e., $\Pa_{\boldsymbol R}(\gr(M), R'_i) = \Pa_{\boldsymbol R}(\gr(M'), R'_i)$ for all $M' \in [M]$ and all $i \in \{ 1, \allowbreak \ldots, \allowbreak k \}$.
	The conditional probability distributions being removed from the product of the interventional distribution thus are identical for all $M' \in [M]$.
	Hence, it remains to be shown that the factorisation of all ground \acp{rv} that are not in $\{ R'_1, \ldots, R'_k \}$ is equivalent for all consistent extensions $M' \in [M]$ of $M$.
	We know that every \ac{pcfg} $M' \in [M]$ entails exactly the same conditional independence statements as $M$ and thus, the factorisation induced by any \ac{pcfg} $M' \in [M]$ is valid, i.e., all \acp{pcfg} $M' \in [M]$ encode the same underlying full joint probability distribution $P$ as $M$.
	Therefore, as the parents of $\{ R'_1, \ldots, R'_k \}$ are identical in all $M' \in [M]$ and all $M' \in [M]$ encode the same probability distribution, the product over the conditional distributions of $R_i \in \{ R_1, \ldots, R_n \} \setminus \{ R'_1, \ldots, R'_k \}$ given their respective parents is identical in all $M' \in [M]$ (just as all \ac{bn} structures over a fixed set of \acp{rv} entailing the same conditional independence statements induce equivalent factorisations of the underlying probability distribution).
	Consequently, the interventional distribution is identical in all consistent extensions $M' \in [M]$ of $M$.
\end{proof}

A direct consequence of \cref{th:pdpcfg_known_parents} is that the result of any interventional query is uniquely determined if there are no undirected edges connected to the intervention variables in the corresponding ground graph of a \ac{pdpcfg}.

\begin{corollary}
	Let $M = (\boldsymbol A \cup \boldsymbol G, \allowbreak \boldsymbol E, \allowbreak \boldsymbol \Phi)$ denote a \ac{pdpcfg}, let $\boldsymbol R = \gr(\boldsymbol A) = \{ Q, \allowbreak R_1, \allowbreak \ldots, \allowbreak R_{\ell}, \allowbreak R'_1, \allowbreak \ldots, \allowbreak R'_k \}$ and let $P(Q \mid do(R'_1 = r'_1, \allowbreak \ldots, \allowbreak R'_k = r'_k))$ be an interventional query.
	If $\Ne_{\boldsymbol R}(\gr(M), R'_1) = \emptyset, \allowbreak \ldots, \allowbreak \Ne_{\boldsymbol R}(\gr(M), R'_k) = \emptyset$, then the result of $P(Q \mid do(R'_1 = r'_1, \allowbreak \ldots, \allowbreak R'_k = r'_k))$ is identical in all consistent extensions $M' \in [M]$ of $M$.
\end{corollary}

Even though in practice, there might be undirected edges connected to the intervention variables, \cref{th:pdpcfg_known_parents} implies that we do not have to consider all possible edge directions of the undirected edges in a \ac{pdpcfg} when computing the effect of an intervention.
Instead, we only have to consider the possible directions of the undirected edges that are relevant for the intervention, i.e., the directions of the undirected edges that are connected to the intervention variables.
Hence, we might not have to consider all consistent extensions of the given \ac{pdpcfg}.
All terms required to answer the interventional query according to the truncated product formula can be computed by querying the \ac{pdpcfg} $M$, as the semantics of $M$ is well-defined even if there are undirected edges in $M$ (that is, the underlying full joint probability distribution is well-defined because its definition is independent of the edge directions in $M$).
Intuitively, it becomes clear that the effect of an intervention is not guaranteed to be uniquely determined if there are undirected edges connected to the intervention variables because there might be various consistent extensions with different parent sets, resulting in multiple possible disjoint effects of the intervention.
This result has been shown for propositional partially directed acyclic graphs~\citep{Maathuis2009a} and we next show that it also holds for \acp{pdpcfg}.

\begin{theorem} \label{th:pdpcfg_unknown_parents}
	Let $M = (\boldsymbol A \cup \boldsymbol G, \allowbreak \boldsymbol E, \allowbreak \boldsymbol \Phi)$ denote a \ac{pdpcfg}, let $\boldsymbol R = \gr(\boldsymbol A) = \{ Q, \allowbreak R_1, \allowbreak \ldots, \allowbreak R_{\ell}, \allowbreak R'_1, \allowbreak \ldots, \allowbreak R'_k \}$ and let $P(Q \mid do(R'_1 = r'_1, \allowbreak \ldots, \allowbreak R'_k = r'_k))$ be an interventional query.
	If there exists a \ac{rv} $R'_i \in \{ R'_1, \ldots, R'_k \}$ such that $\Ne_{\boldsymbol R}(\gr(M), R'_i) \neq \emptyset$, then there might be consistent extensions $M_1, M_2 \in [M]$ of $M$ such that the result of $P(Q \mid do(R'_1 = r'_1, \allowbreak \ldots, \allowbreak R'_k = r'_k))$ is not identical in $M_1$ and $M_2$.
\end{theorem}
\begin{proof}
	If there exists a \ac{rv} $R'_i \in \{R'_1, \ldots, R'_k\}$ such that $\Ne_{\boldsymbol R}(\gr(M), R'_i) \neq \emptyset$ holds, there might exist $M_1, M_2 \in [M]$ such that $\Pa_{\boldsymbol R}(\gr(M_1), R'_i) \neq \Pa_{\boldsymbol R}(\gr(M_2), R'_i)$.
	Then, by definition of the interventional distribution (\cref{def:pcfg_interventional_distribution_pcfg}), the conditional probability distributions being removed from the product differ in $M_1$ and $M_2$, thereby 
	yielding different interventional distributions for $M_1$ and $M_2$.
\end{proof}

Generally, there might be scenarios in which it is possible to uniquely determine the result of an interventional query even if there are undirected edges connected to the intervention variables, as possibly not all undirected edges can be oriented in both directions.
In particular, some orientations might introduce a cycle or change the conditional independence statements implied by the graph structure and hence do not result in a consistent extension.
In other words, it might be possible that there is just a single possible orientation of the parents of the intervention variables and in this case, the result of any interventional query can be uniquely determined.

Next, we gather the theoretical insights from this section to introduce the \ac{elci} algorithm, which efficiently computes the effect of an intervention in a \ac{pdpcfg}.

\section{The Extended Lifted Causal Inference Algorithm} \label{sec:pdpcfg_lci}
Combining the insights from \cref{th:pdpcfg_known_parents,th:pdpcfg_unknown_parents} naturally leads to an algorithm to compute the effect of interventions in a \ac{pdpcfg}.
The idea is that all possible parent sets of the intervention variables have to be considered.
If there is just one possible set of parents, the effect of the intervention can be uniquely determined, otherwise there are multiple possible effects that are enumerated.
This idea is incorporated in the IDA algorithm and its variants~\citep{Guo2021a,Liu2020a,Maathuis2009a} for interventions $do(R' = r')$ with a single intervention variable $R'$ in propositional causal models.
\citet{Nandy2017a} also consider the case of multiple intervention variables $R'_1, \ldots, R'_k$ in propositional causal models, however, they operate in a different setting as they assume observational data generated by an unknown linear structural equation model with independent errors.
\Cref{alg:pdpcfg_elci} displays the \ac{elci} algorithm, which extends the idea of just considering the possible parent sets of intervention variables to handle arbitrary interventions $do(R'_1 = r'_1, \ldots, R'_k = r'_k)$ with $k \geq 1$ in a \ac{pdpcfg}.

\begin{algorithm}[t]
	\caption{Extended Lifted Causal Inference}
	\label{alg:pdpcfg_elci}
	\alginput{A \ac{pdpcfg} $M = (\boldsymbol A \cup \boldsymbol G, \boldsymbol E, \boldsymbol \Phi)$, and an interventional query $P(Q \mid do(R'_1 = r'_1, \allowbreak \ldots, \allowbreak R'_k = r'_k))$ with $\{Q, \allowbreak R'_1, \allowbreak \ldots, \allowbreak R'_k\} \subseteq \boldsymbol R = \gr(\boldsymbol A) = \{ Q, \allowbreak R_1, \allowbreak \ldots, \allowbreak R_{\ell}, \allowbreak R'_1, \allowbreak \ldots, \allowbreak R'_k \}$.} \\
	\algoutput{The set of all possible answers to the interventional query $P(Q \mid do(R'_1 = r'_1, \allowbreak \ldots, \allowbreak R'_k = r'_k))$ in $M$.}
	\begin{algorithmic}[1]
		\State $\boldsymbol P \gets \emptyset$\; \label{line:pdpcfg_elci_init}
		\State $M \gets$ \Ac{pdpcfg} after splitting \acp{pf} in $M$ on each $R'_i \in \{ R'_1, \ldots, R'_k \}$\; \label{line:pdpcfg_elci_split}
		\ForEach{$\boldsymbol C_1 \subseteq \Ne_{\boldsymbol R}(M, R'_1), \ldots, \boldsymbol C_k \subseteq \Ne_{\boldsymbol R}(M, R'_k)$ s.t.\ $\boldsymbol C_1, \ldots, \boldsymbol C_k$ are cliques} \label{line:pdpcfg_elci_outer_loop}
			\State $M' \gets M$\; \label{line:pdpcfg_elci_init_modified_model}
			\ForEach{intervention variable $R'_i \in \{ R'_1, \ldots, R'_k \}$} \label{line:pdpcfg_elci_middle_loop}
				\ForEach{undirected neighbour \ac{rv} $C \in \boldsymbol C_i$ of $R'_i$} \label{line:pdpcfg_elci_inner_loop}
					\State Orient $C - f - R'_i$ as $C - f \to R'_i$ in $M'$\; \label{line:pdpcfg_elci_orient_parents}
				\EndForEach
			\EndForEach
			\If{$[M'] = \emptyset$} \label{line:pdpcfg_elci_if_no_extension}
				\State \Continue\; \label{line:pdpcfg_elci_if_no_extension_continue}
			\EndIf
			\State $M'' \gets$ Any consistent extension from $[M']$\; \label{line:pdpcfg_elci_consistent_extension}
			\State $D \gets \sum\limits_{r_1 \in \range{R_1}} \ldots \sum\limits_{r_{\ell} \in \range{R_{\ell}}} \prod\limits_{R_i \in \{ Q, R_1, \ldots, R_{\ell} \}} P(r_i \mid \pa_{\boldsymbol R}(M'', R_i))$\; \label{line:pdpcfg_elci_formula}
			\State Add $D$ to $\boldsymbol P$\; \label{line:pdpcfg_elci_add_result}
		\EndForEach
		\State \Return $\boldsymbol P$\; \label{line:pdpcfg_elci_return}
	\end{algorithmic}
\end{algorithm}

Given a \ac{pdpcfg} $M = (\boldsymbol A \cup \boldsymbol G, \boldsymbol E, \boldsymbol \Phi)$ and an interventional query $P(Q \mid do(R'_1 = r'_1, \allowbreak \ldots, \allowbreak R'_k = r'_k))$, \ac{elci} proceeds as follows to compute the set of all possible results for $P(Q \mid do(R'_1 = r'_1, \allowbreak \ldots, \allowbreak R'_k = r'_k))$.
First, after initialising an empty set $\boldsymbol P$ to which possible query results are added (\cref{line:pdpcfg_elci_init}), \ac{elci} splits the \acp{pf} in $M$ based on each $R'_i \in \{R'_1, \ldots, R'_k\}$ (\cref{line:pdpcfg_elci_split}).
In particular, \ac{elci} splits every \ac{pf} $\phi \in \boldsymbol \Phi$ for which there is an instance $\phi_j \in \gr(\phi)$ such that any intervention variable $R'_i \in \{R'_1, \allowbreak \ldots, \allowbreak R'_k\}$ is a child of $\phi_j$.
\Ac{elci} then iterates over all possible combinations of parent sets (i.e., over all combinations of subsets of undirected neighbours) of the intervention variables $R'_1, \ldots, R'_k$ (\cref{line:pdpcfg_elci_outer_loop,line:pdpcfg_elci_init_modified_model,line:pdpcfg_elci_middle_loop,line:pdpcfg_elci_inner_loop,line:pdpcfg_elci_orient_parents}).
When considering the subsets of undirected neighbours, it is necessary that all subsets are jointly valid, that is, they are not allowed to alter the conditional independence statements encoded by the model and they must not introduce any directed cycles when oriented towards $R'_1, \ldots, R'_k$.
To ensure the validity of these subsets, they are required to form a clique.
A clique $\boldsymbol C$ is a subset of nodes such that all pairs of nodes in $\boldsymbol C$ are directly connected via a \ac{pf} node, that is, for each pair of nodes $C_1 \in \boldsymbol C$, $C_2 \in \boldsymbol C$ with $C_1 \neq C_2$ it holds that there exists a \ac{pf} node $g \in \boldsymbol G$ such that there is an edge between $C_1$ and $g$ as well as an edge between $C_2$ and $g$ in $\boldsymbol E$ (either directed or undirected).
By ensuring that the subsets of undirected neighbours form cliques, the orientation of the incident edges towards $R'_1, \ldots, R'_k$ does not introduce any pattern $C_1 - g_1 \to R'_i \gets g_2 - C_2$ where $C_1$ and $C_2$ are not directly connected via a \ac{pf} node, as due to the clique property, $C_1$ and $C_2$ are always guaranteed to be directly connected via a factor node.
In consequence, the conditional independence statements encoded by $M'$ are guaranteed to be equivalent to those encoded by $M$~\citep{Maathuis2009a}.
Having obtained a possible combination of parent sets of $R'_1, \ldots, R'_k$, \ac{elci} next extends the modified model $M'$ to any \ac{pcfg} from the set of consistent extensions of $M'$, if such a consistent extension exists (\cref{line:pdpcfg_elci_consistent_extension,line:pdpcfg_elci_if_no_extension,line:pdpcfg_elci_if_no_extension_continue}).
In case there is no consistent extension (e.g., due to $M'$ containing a directed cycle), \ac{elci} continues with the next possible combination of parent sets.
If there is a consistent extension $M'' \in [M']$, then the result of the provided query is given by applying the truncated product formula (\cref{eq:pcfg_truncated_product_formula_pcfg}) in $M''$.
As the parents of the intervention variables are fixed in $M'$, the result of the given interventional query is identical in all consistent extensions of $M'$ according to \cref{th:pdpcfg_known_parents}.
Thus, \ac{elci} computes the result of the interventional query in $M''$ by applying the truncated product formula and then adds the result to $\boldsymbol P$ (\cref{line:pdpcfg_elci_formula,line:pdpcfg_elci_add_result}).
We remark that the computation can be simplified if it is known that the factors in $M$ encode conditional probability distributions.
Then, the modification from \cref{prop:pcfg_lci_intervention} can be applied as in the \ac{lci} algorithm (\cref{alg:pcfg_lci}).
After the result of the interventional query in $M''$ has been computed, \ac{elci} repeats the above steps for the next parent set of the intervention variables $R'_1, \allowbreak \ldots, \allowbreak R'_k$ until all possible parent sets have been taken into account.
In the end, \ac{elci} returns the set $\boldsymbol P$ containing all possible results for the interventional query $P(Q \mid do(R'_1 = r'_1, \allowbreak \ldots, \allowbreak R'_k = r'_k))$ (\cref{line:pdpcfg_elci_return}).
In case there is no causal explanation for the given \ac{pdpcfg} $M$ (that is, $M$ has no consistent extension at all), \ac{elci} returns an empty set.
Such a situation might occur if $M$ already contains a directed cycle or if there are undirected edges that cannot be oriented without introducing a directed cycle or altering the conditional independence statements implied by the model.

\Ac{elci} is also able to handle multiple query variables at once (then, \cref{line:pdpcfg_elci_formula} is adjusted such that only non-query variables are summed out).
To compute a consistent extension in \cref{line:pdpcfg_elci_consistent_extension}, \ac{elci} might just call any of the efficient extension algorithms that are already available~\citep{Verma1992a,Wienoebst2021a,Luttermann2023a}.
Each of these algorithms operates on a partially directed acyclic graph and hence can be directly applied to the underlying causal graph of the \ac{pdpcfg}.
While the set of probabilistic queries obtained from the truncated product formula refers to ground \acp{rv}, these queries can be answered using lifted probabilistic inference, e.g., by running the \ac{ljt} algorithm (which is specifically designed to efficiently handle sets of queries) on $M$.
The correctness of \ac{elci} directly follows from \cref{th:pdpcfg_known_parents}.

\begin{proposition}
	The result computed by \ac{elci} (\cref{alg:pdpcfg_elci}) is correct, i.e., the returned set contains all possible results for the interventional query $P(Q \mid do(R'_1 = r'_1, \allowbreak \ldots, \allowbreak R'_k = r'_k))$ in the given \ac{pdpcfg} $M$.
\end{proposition}
\begin{proof}
	From \cref{th:pdpcfg_known_parents}, we know that the result of any interventional query is identical in all consistent extensions of a \ac{pdpcfg} $M$ if the parents of the intervention variables are known.
	Thus, to compute the set of possible results for $P(Q \mid do(R'_1 = r'_1, \allowbreak \ldots, \allowbreak R'_k = r'_k))$ in $M$, it is sufficient to consider all possible parent sets of the intervention variables $R'_1, \allowbreak \ldots, \allowbreak R'_k$ and compute the result of the given query in the resulting models.
	Due to \citet{Maathuis2009a}, it holds that any subset of undirected neighbours of the intervention variables needs to form a clique in order to obtain a valid orientation when orienting the edges towards the intervention variables (as otherwise, the conditional independence statements induced by the graph change). 
	Consequently, by ensuring that undirected neighbours of $R'_1, \allowbreak \ldots, \allowbreak R'_k$ form cliques before orienting them towards $R'_1, \allowbreak \ldots, \allowbreak R'_k$, \ac{elci} does not miss any possible parent set of the intervention variables.
	For any fixed set of parents of $R'_1, \allowbreak \ldots, \allowbreak R'_k$, due to \cref{th:pdpcfg_known_parents} it is then sufficient to consider any consistent extension and compute the result for the given query in it.
	As \ac{elci} applies the truncated product formula from \cref{def:pcfg_truncated_product_formula_pcfg} to compute the result of the given query, the correctness of \ac{elci} follows.
\end{proof}

Given our assumption that the graph structure is identical for all groundings, it also holds that, e.g., given an intervention $do(Com(E) = \high)$, \ac{elci} has to consider only two possible parent sets regardless of the number of employees while there are $2^{\abs{\domain{E}}}$ possible parent sets in an equivalent propositional model to consider.
In a propositional model, it is also possible to reduce the number of possible parent sets when background knowledge is introduced, i.e., when knowing that specific \acp{rv} are actually representable by a single \ac{prv}.

\begin{corollary}
	Let $M = (\boldsymbol A \cup \boldsymbol G, \boldsymbol E, \boldsymbol \Phi)$ be a \ac{pdpcfg}.
	When intervening on a \ac{prv} $A(\mathcal L)_{| C} \in \boldsymbol A$, under the assumption that the graph structure is identical for all groundings, it holds that
	\begin{enumerate}
		\item \ac{elci} considers $O(2^{\abs{\Ne_{\boldsymbol A}(M, A)}})$ possible parent sets in the worst case, and
		\item in a propositional model, $O(2^{\sum_{R \in \gr(A)} \abs{\Ne_{\boldsymbol R}(\gr(M), R)}})$ possible parent sets have to be considered in the worst case.
	\end{enumerate}
\end{corollary}

We refrain from empirically evaluating \ac{elci} as we have already shown the superiority of \ac{lci} to the propositional case.
\Ac{elci} (\ac{lci}, respectively) enables the computation of answers to interventional queries on a lifted level and hence can also be plugged into parameterised decision models~\citep{Gehrke2018b} to compute the action that maximises the expected utility under the semantics of interventions (instead of using the semantics of conditioning).
A parameterised decision model originally extends a \ac{pfg} by action nodes and utility nodes.
Instead of using an undirected \ac{pfg} as a basis, we can use a \ac{pdpcfg} (or a \ac{pcfg}) as a basis for a parameterised decision model and then compute the expected utility of an action using \ac{elci} (\ac{lci}, respectively), thereby allowing for first-order decision making.

\section{Conclusion} \label{sec:pdpcfg_conclusion}
We introduce \acp{pcfg} to combine lifted probabilistic inference with causal knowledge.
To leverage the power of lifted inference for the computation of the effect of interventions, we further present the \ac{lci} algorithm, which operates on a lifted level and thus allows us to drastically speed up causal inference compared to running causal inference on an equivalent propositional (ground) model.
\Ac{lci} is a simple, yet effective algorithm to compute the effect of interventions.
Moreover, we introduce \acp{pdpcfg} as lifted causal models that allow to incorporate partial causal knowledge, thereby enabling lifted causal inference without the requirement of having a fully specified causal graph at hand.
A \ac{pdpcfg} generalises the concept of a \ac{pcfg} by incorporating both undirected and directed edges in the graph structure.
To compute the effect of interventions in a \ac{pdpcfg}, we introduce the \ac{elci} algorithm, which enumerates all possible results for an interventional query without grounding the entire model.

An interesting direction for future work is to relax the causal sufficiency assumption (\cref{def:pcfg_causal_sufficiency}) and allow for hidden confounders (i.e., confounding variables that are not observed and hence not included in the set of \acp{rv} over which the model is defined) in a \ac{pdpcfg}.
Under the presence of hidden confounders, the result of an interventional query might not be uniquely determinable anymore.
Thus, the identification problem (i.e., the problem of determining whether a causal effect can be uniquely identified) becomes relevant if hidden confounders are present.
In particular, an interesting question is whether the $do$-calculus introduced by \citet{Pearl1995a}, which allows to rewrite an interventional query to obtain a probabilistic query free of $do$-expressions, can be applied to a \ac{pdpcfg} with hidden confounders.

\backmatter

\bmhead{Acknowledgements}
This work is supported by the BMBF project AnoMed 16KISA057 and extends the works \citep{Luttermann2024b} and \citep{Luttermann2024g}.
The authors also would like to thank the anonymous reviewers for their valuable feedback and suggestions to improve the manuscript.
This version of the article has been accepted for publication, after peer review but is not the Version of Record and does not reflect post-acceptance improvements, or any corrections.
The Version of Record is available online at: \url{https://doi.org/10.1007/s10472-026-10009-1}.

\bmhead{Data Availability Statement}
The source code including the data set generators used for the experiments in this article is available at \url{https://github.com/StatisticalRelationalAI/LiftedCausalInference}.

\bmhead{Conflict of Interest}
The authors declare that they have no conflict of interest.

\bibliography{references.bib}

@article{Pearl1986a,
	author    = {Judea Pearl},
	title     = {{Fusion, Propagation, and Structuring in Belief Networks}},
	journal   = {{Artificial Intelligence}},
	volume    = {29},
	year      = {1986},
	pages     = {241--288},
	publisher = {{Elsevier}},
}

@book{Pearl1988a,
	author    = {Judea Pearl},
	title     = {{Probabilistic Reasoning in Intelligent Systems: Networks of Plausible Inference}},
	year      = {1988},
	publisher = {{Morgan Kaufmann Publishers Inc.}},
}

@inproceedings{Verma1990a,
	author    = {Thomas Verma and Judea Pearl},
	title     = {{Equivalence and Synthesis of Causal Models}},
	booktitle = {{Proceedings of the Sixth Conference on Uncertainty in Artificial Intelligence (UAI-1990)}},
	year      = {1990},
	pages     = {255--270},
	publisher = {{Elsevier}},
}

@inproceedings{Verma1992a,
	author    = {Thomas Verma and Judea Pearl},
	title     = {{An Algorithm for Deciding if a Set of Observed Independencies
	Has a Causal Explanation}},
	booktitle = {{Proceedings of the Eighth Conference on Uncertainty in Artificial Intelligence (UAI-1992)}},
	year      = {1992},
	pages     = {323--330},
	publisher = {{Morgan Kaufmann Publishers Inc.}},
}

@article{Pearl1995a,
	author    = {Judea Pearl},
	title     = {{Causal Diagrams for Empirical Research}},
	journal   = {{Biometrika}},
	volume    = {82},
	year      = {1995},
	pages     = {669--688},
	publisher = {{Oxford University Press}},
}

@book{Lauritzen1996a,
	author    = {Steffen L. Lauritzen},
	title     = {{Graphical Models}},
	year      = {1996},
	publisher = {{Clarendon Press}},
}

@article{Andersson1997a,
	author    = {Steen A. Andersson and David Madigan and Michael D. Perlman},
	title     = {{A Characterization of Markov Equivalence Classes for Acyclic Digraphs}},
	journal   = {{The Annals of Statistics}},
	volume    = {25},
	year      = {1997},
	pages     = {505--541},
	publisher = {{Institute of Mathematical Statistics}},
}

@inproceedings{Frey1997a,
	author    = {Brendan J. Frey and Frank R. Kschischang and Hans-Andrea Loeliger and Niclas Wiberg},
	title     = {{Factor Graphs and Algorithms}},
	booktitle = {{Proceedings of the Thirty-Fifth Annual Allerton Conference on Communication, Control, and Computing}},
	year      = {1997},
	pages     = {666--680},
	publisher = {{Allerton House}},
}

@inproceedings{Shachter1998a,
	author    = {Ross D. Shachter},
	title     = {{Bayes-Ball: Rational Pastime (For Determining Irrelevance and Requisite Information in Belief Networks and Influence Diagrams)}},
	booktitle = {{Proceedings of the Fourteenth Conference on Uncertainty in Artificial Intelligence (UAI-1998)}},
	year      = {1998},
	pages     = {480--487},
	publisher = {{Morgan Kaufmann Publishers Inc.}},
}

@book{Spirtes2000a,
	author    = {Peter Spirtes and Clark Glymour and Richard Scheines},
	title     = {{Causation, Prediction, and Search}},
	year      = {2000},
	publisher = {{MIT Press}},
	edition   = {2nd},
}

@article{Kschischang2001a,
	author    = {Frank R. Kschischang and Brendan J. Frey and Hans-Andrea Loeliger},
	title     = {{Factor Graphs and the Sum-Product Algorithm}},
	journal   = {{IEEE Transactions on Information Theory}},
	volume    = {47},
	year      = {2001},
	pages     = {498--519},
	publisher = {{IEEE}},
}

@inproceedings{Frey2003a,
	author    = {Brendan J. Frey},
	title     = {{Extending Factor Graphs so as to Unify Directed and Undirected Graphical Models}},
	booktitle = {{Proceedings of the Nineteenth Conference on Uncertainty in Artificial Intelligence (UAI-2003)}},
	year      = {2003},
	pages     = {257--264},
	publisher = {{Morgan Kaufmann Publishers Inc.}},
}

@inproceedings{Poole2003a,
	author    = {David Poole},
	title     = {{First-Order Probabilistic Inference}},
	booktitle = {{Proceedings of the Eighteenth International Joint Conference on Artificial Intelligence (IJCAI-2003)}},
	year      = {2003},
	pages     = {985--991},
	publisher = {{Morgan Kaufmann Publishers Inc.}},
}

@inproceedings{DeSalvoBraz2005a,
	author    = {Rodrigo {De Salvo Braz} and Eyal Amir and Dan Roth},
	title     = {{Lifted First-Order Probabilistic Inference}},
	booktitle = {{Proceedings of the Nineteenth International Joint Conference on Artificial Intelligence (IJCAI-2005)}},
	year      = {2005},
	pages     = {1319--1325},
	publisher = {{Morgan Kaufmann Publishers Inc.}},
}

@inproceedings{DeSalvoBraz2006a,
	author    = {Rodrigo {De Salvo Braz} and Eyal Amir and Dan Roth},
	title     = {{MPE and Partial Inversion in Lifted Probabilistic Variable Elimination}},
	booktitle = {{Proceedings of the Twenty-First National Conference on Artificial Intelligence (AAAI-2006)}},
	year      = {2006},
	pages     = {1123--1130},
	publisher = {{AAAI Press}},
}

@inproceedings{Milch2008a,
	author    = {Brian Milch and Luke S. Zettlemoyer and Kristian Kersting and Michael Haimes and Leslie Pack Kaelbling},
	title     = {{Lifted Probabilistic Inference with Counting Formulas}},
	booktitle = {{Proceedings of the Twenty-Third AAAI Conference on Artificial Intelligence (AAAI-2008)}},
	year      = {2008},
	pages     = {1062--1068},
	publisher = {{AAAI Press}},
}

@inproceedings{Kisynski2009a,
	author    = {Jacek Kisy\'{n}ski and David Poole},
	title     = {{Constraint Processing in Lifted Probabilistic Inference}},
	booktitle = {{Proceedings of the Twenty-Fifth Conference on Uncertainty in Artificial Intelligence (UAI-2009)}},
	year      = {2009},
	pages     = {293--302},
	publisher = {{AUAI Press}},
}

@article{Maathuis2009a,
	author    = {Marloes H. Maathuis and Markus Kalisch and Peter Bühlmann},
	title     = {{Estimating High-Dimensional Intervention Effects from Observational Data}},
	journal   = {{The Annals of Statistics}},
	volume    = {37},
	year      = {2009},
	pages     = {3133--3164},
	publisher = {{Institute of Mathematical Statistics}},
}

@book{Pearl2009a,
	author    = {Judea Pearl},
	title     = {{Causality: Models, Reasoning and Inference}},
	year      = {2009},
	publisher = {{Cambridge University Press}},
	edition   = {2nd},
}

@inproceedings{Maier2010a,
	author    = {Marc Maier and Brian Taylor and Huseyin Oktay and David Jensen},
	title     = {{Learning Causal Models of Relational Domains}},
	booktitle = {{Proceedings of the Twenty-Fourth AAAI Conference on Artificial Intelligence (AAAI-2010)}},	
	year      = {2010},
	pages     = {531--538},
	publisher = {{AAAI Press}},
}

@inproceedings{Meert2010a,
	author    = {Wannes Meert and Nima Taghipour and Hendrik Blockeel},
	title     = {{First-Order Bayes-Ball}},
	booktitle = {{Proceedings of the Fourteenth European Conference on Machine Learning and Principles and Practice of Knowledge Discovery in Databases (ECML PKDD-2010)}},
	year      = {2010},
	pages     = {369--384},
	publisher = {{Springer}},
}

@inproceedings{VanDenBroeck2011a,
	author    = {Guy {Van den Broeck}},
	title     = {{On the Completeness of First-Order Knowledge Compilation for Lifted Probabilistic Inference}},
	booktitle = {{Advances in Neural Information Processing Systems 24 (NIPS-2011)}},
	year      = {2011},
	pages     = {1386--1394},
	publisher = {{Curran Associates, Inc.}},
}

@inproceedings{Apsel2012a,
	author    = {Udi Apsel and Ronen I. Brafman},
	title     = {{Lifted MEU by Weighted Model Counting}},
	booktitle = {{Proceedings of the Twenty-Sixth AAAI Conference on Artificial Intelligence (AAAI-2012)}},
	year      = {2012},
	pages     = {1861--1867},
	publisher = {{AAAI Press}},
}

@inproceedings{Winn2012a,
	author    = {John Winn},
	title     = {{Causality with Gates}},
	booktitle = {{Proceedings of the Fifteenth International Conference on Artificial Intelligence and Statistics (AISTATS-2012)}},
	year      = {2012},
	pages     = {1314--1322},
	publisher = {{PMLR}},
}

@inproceedings{Maier2013a,
	author    = {Marc Maier and Katerina Marazopoulou and David Arbour and David Jensen},
	title     = {{A Sound and Complete Algorithm for Learning Causal Models from Relational Data}}, 
	booktitle = {{Proceedings of the Twenty-Ninth Conference on Uncertainty in Artificial Intelligence (UAI-2013)}},
	year      = {2013},
	pages     = {371--380},
	publisher = {{AUAI Press}},
}

@article{Taghipour2013a,
	author    = {Nima Taghipour and Daan Fierens and Jesse Davis and Hendrik Blockeel},
	title     = {{Lifted Variable Elimination: Decoupling the Operators from the Constraint Language}},
	journal   = {{Journal of Artificial Intelligence Research}},
	volume    = {47},
	year      = {2013},
	pages     = {393--439},
	publisher = {{AI Access Foundation}},
}

@inproceedings{Taghipour2013b,
	author    = {Nima Taghipour and Daan Fierens and Guy {Van den Broeck} and Jesse Davis and Hendrik Blockeel},
	title     = {{Completeness Results for Lifted Variable Elimination}},
	booktitle = {{Proceedings of the Sixteenth International Conference on Artificial Intelligence and Statistics (AISTATS-2013)}},
	year      = {2013},
	pages     = {572--580},
	publisher = {{PMLR}},
}

@inproceedings{Niepert2014a,
	author    = {Mathias Niepert and Guy {Van den Broeck}},
	title     = {{Tractability through Exchangeability: A New Perspective on Efficient Probabilistic Inference}},
	booktitle = {{Proceedings of the Twenty-Eighth AAAI Conference on Artificial Intelligence (AAAI-2014)}},
	year      = {2014},
	pages     = {2467--2475},
	publisher = {{AAAI Press}},
}

@inproceedings{Lee2015a,
	author    = {Sanghack Lee and Vasant Honavar},
	title     = {{Lifted Representation of Relational Causal Models Revisited: Implications for Reasoning and Structure Learning}}, 
	booktitle = {{Proceedings of the UAI 2015 Conference on Advances in Causal Inference}},
	year      = {2015},
	pages     = {56--65},
	publisher = {{CEUR}},
}

@inproceedings{Arbour2016a,
	author    = {David Arbour and Dan Garant and David Jensen},
	title     = {{Inferring Network Effects from Observational Data}},
	booktitle = {{Proceedings of the Twenty-Second ACM SIGKDD International Conference on Knowledge Discovery and Data Mining (KDD-2016)}},
	year      = {2016},
	pages     = {715--724},
	publisher = {{ACM Press}},
}

@inproceedings{Braun2016a,
	author    = {Tanya Braun and Ralf M\"oller},
	title     = {{Lifted Junction Tree Algorithm}},
	booktitle = {{Proceedings of the Thirty-Ninth German Conference on Artificial Intelligence (KI-2016)}},
	year      = {2016},
	pages     = {30--42},
	publisher = {{Springer}},
}

@inproceedings{Lee2016a,
	author    = {Sanghack Lee and Vasant Honavar},
	title     = {{On Learning Causal Models from Relational Data}},
	booktitle = {{Proceedings of the Thirtieth AAAI Conference on Artificial Intelligence (AAAI-2016)}},
	year      = {2016},
	pages     = {3263--3270},
	publisher = {{AAAI Press}},
}

@book{Pearl2016a,
	author    = {Judea Pearl and Madelyn Glymour and Nicholas P. Jewell},
	title     = {{Causal Inference in Statistics: A Primer}},
	year      = {2016},
	publisher = {{Wiley}},
	edition   = {1st},
}

@inbook{Genesereth2017a,
	author    = {Michael Genesereth and Eric J. Kao},
	title     = {{Relational Logic}},
	booktitle = {{Introduction to Logic}},
	year      = {2017},
	pages     = {63--81},
	publisher = {{Springer}},
}

@article{Nandy2017a,
	author    = {Preetam Nandy and Marloes H. Maathuis and Thomas S. Richardson},
	title     = {{Estimating the Effect of Joint Interventions from Observational Data in Sparse High-Dimensional Settings}},
	journal   = {{The Annals of Statistics}},
	volume    = {45},
	year      = {2017},
	pages     = {647--674},
	publisher = {{Institute of Mathematical Statistics}},
}

@book{Peters2017a,
	author    = {Jonas Peters and Dominik Janzing and Bernhard Schölkopf},
	title     = {{Elements of Causal Inference: Foundations and Learning Algorithms}},
	year      = {2017},
	publisher = {{MIT Press}},
}

@inproceedings{Braun2018a,
	author    = {Tanya Braun and Ralf Möller},
	title     = {{Parameterised Queries and Lifted Query Answering}},
	booktitle = {{Proceedings of the Twenty-Seventh International Joint Conference on Artificial Intelligence (IJCAI-2018)}},
	year      = {2018},
	pages     = {4980--4986},
	publisher = {{IJCAI Organization}},
}

@inproceedings{Gehrke2018a,
	author    = {Marcel Gehrke and Tanya Braun and Ralf M\"oller},
	title     = {{Lifted Dynamic Junction Tree Algorithm}},
	booktitle = {{Proceedings of the Twenty-Third International Conference on Conceptual Structures (ICCS-2018)}},
	year      = {2018},
	pages     = {55--69},
	publisher = {{Springer}},
}

@inproceedings{Gehrke2018b,
	author    = {Marcel Gehrke and Tanya Braun and Ralf M\"oller and Alexander Waschkau and Christoph Strumann and Jost Steinh\"auser},
	title     = {{Lifted Maximum Expected Utility}},
	booktitle = {{Proceedings of the First International Workshop on Artificial Intelligence in Health (AIH-2018)}},
	year      = {2019},
	pages     = {131--141},
	publisher = {{Springer}},
}

@inproceedings{Gehrke2019c,
	author    = {Marcel Gehrke and Tanya Braun and Ralf M\"oller}, 
	title     = {{Lifted Temporal Maximum Expected Utility}}, 
	booktitle = {{Proceedings of the Thirty-Second Canadian Conference on Artificial Intelligence (CANAI-2019)}},
	year      = {2019},
	pages     = {380--386},
	publisher = {{Springer}},
}

@inproceedings{Lee2019a,
	author    = {Sanghack Lee and Vasant Honavar},
	title     = {{Towards Robust Relational Causal Discovery}},
	booktitle = {{Proceedings of The Thirty-Fifth Conference on Uncertainty in Artificial Intelligence (UAI-2019)}},
	year      = {2019},
	pages     = {345--355},
	publisher = {{PMLR}},
}

@phdthesis{Braun2020a,
	author = {Tanya Braun},
	title  = {{Rescued from a Sea of Queries: Exact Inference in Probabilistic Relational Models}},
	school = {{University of Lübeck}},
	year   = {2020},
}

@inproceedings{Gehrke2020a,
	author    = {Marcel Gehrke and Ralf M\"oller and Tanya Braun}, 
	title     = {{Taming Reasoning in Temporal Probabilistic Relational Models}}, 
	booktitle = {{Proceedings of the Twenty-Fourth European Conference on Artificial Intelligence (ECAI-2020)}}, 
	year      = {2020},
	pages     = {2592--2599},
	publisher = {{IOS Press}},
}

@inproceedings{Liu2020a,
	author    = {Liu, Yue and Fang, Zhuangyan and He, Yangbo and Geng, Zhi},
	title     = {{Collapsible IDA: Collapsing Parental Sets for Locally Estimating Possible Causal Effects}},
	booktitle = {{Proceedings of the Thirty-Sixth Conference on Uncertainty in Artificial Intelligence (UAI-2020)}},
	year      = {2020},
	pages     = {290--299},
	publisher = {{PMLR}},
}

@inproceedings{Salimi2020a,
	author    = {Babak Salimi and Harsh Parikh and Moe Kayali and Lise Getoor and Sudeepa Roy and Dan Suciu},
	title     = {{Causal Relational Learning}},
	booktitle = {{Proceedings of the 2020 ACM SIGMOD International Conference on Management of Data}},
	year      = {2020},
	pages     = {241--256},
	publisher = {{ACM Press}},
}

@inproceedings{Guo2021a,
	author    = {Richard Guo and Emilija Perkovic},
	title     = {{Minimal Enumeration of All Possible Total Effects in a Markov Equivalence Class}},
	booktitle = {{Proceedings of The Twenty-Fourth International Conference on Artificial Intelligence and Statistics (AISTATS-2021)}},
	year      = {2021},
	pages     = {2395--2403},
	publisher = {{PMLR}},
}

@inproceedings{Wienoebst2021a,
	author    = {Marcel Wienöbst and Max Bannach and Maciej Li\'{s}kiewicz},
	title     = {{Extendability of Causal Graphical Models: Algorithms and Computational Complexity}},
	booktitle = {{Proceedings of the Thirty-Seventh Conference on Uncertainty in Artificial Intelligence (UAI-2021)}},
	year      = {2021},
	pages     = {1248--1257},
	publisher = {{PMLR}},
}

@inproceedings{Ahsan2022a,
	author    = {Ragib Ahsan and David Arbour and Elena Zheleva},
	title     = {{Relational Causal Models with Cycles: Representation and Reasoning}},
	booktitle = {{Proceedings of the First Conference on Causal Learning and Reasoning (CLeaR-2022)}},
	year      = {2022},
	pages     = {1--18},
	publisher = {{PMLR}},
}

@inproceedings{Braun2022a,
	author    = {Tanya Braun and Marcel Gehrke},
	title     = {{Explainable and Explorable Decision Support}},
	booktitle = {{Proceedings of the Twenty-Seventh International Conference on Conceptual Structures (ICCS-2022)}},
	year      = {2022},
	pages     = {99--114},
	publisher = {{Springer}},
}

@inproceedings{Ahsan2023a,
	author    = {Ragib Ahsan and David Arbour and Elena Zheleva},
	title     = {{Learning Relational Causal Models with Cycles through Relational Acyclification}},
	booktitle = {{Proceedings of the Thirty-Seventh AAAI Conference on Artificial Intelligence (AAAI-2023)}},
	year      = {2023},
	pages     = {12164--12171},
	publisher = {{AAAI Press}},
}

@inproceedings{Luttermann2023a,
	author    = {Malte Luttermann and Marcel Wien\"obst and Maciej Li\'skiewicz},
	title     = {{Practical Algorithms for Orientations of Partially Directed Graphical Models}},
	booktitle = {{Proceedings of the Second Conference on Causal Learning and Reasoning (CLeaR-2023)}},
	year      = {2023},
	pages     = {259--280},
	publisher = {{PMLR}},
}

@inproceedings{Luttermann2024b,
	author    = {Malte Luttermann and Mattis Hartwig and Tanya Braun and Ralf M\"oller and Marcel Gehrke},
	title     = {{Lifted Causal Inference in Relational Domains}},
	booktitle = {{Proceedings of the Third Conference on Causal Learning and Reasoning (CLeaR-2024)}},
	year      = {2024},
	pages     = {827--842},
	publisher = {{PMLR}},
}

@inproceedings{Luttermann2024g,
	author    = {Malte Luttermann and Tanya Braun and Ralf Möller and Marcel Gehrke},
	title     = {{Estimating Causal Effects in Partially Directed Parametric Causal Factor Graphs}},
	booktitle = {{Proceedings of the Sixteenth International Conference on Scalable Uncertainty Management (SUM-2024)}},
	year      = {2024},
	pages     = {265--280},
	publisher = {{Springer}},
}
 
\end{document}